%% file: camera_ready.tex
\documentclass[twoside]{article}

\input{math_commands.tex}

\usepackage{xr-hyper}
\usepackage{hyperref}
\usepackage{xr}

\usepackage[capitalize, noabbrev]{cleveref}
\usepackage{url}

\usepackage{microtype}
\usepackage{graphicx}
\usepackage{bbding}
\usepackage{subfigure}

\usepackage{amsmath}
\usepackage{amssymb}
\usepackage{mathtools}
\usepackage{amsthm}
\usepackage{bbm}
\usepackage{dsfont}
\usepackage{booktabs} 
\usepackage{multirow}
\usepackage{array}
\usepackage{xcolor}
\usepackage{color}
\usepackage{balance}

\theoremstyle{plain}
\newtheorem{theorem}{Theorem}[section]
\newtheorem{proposition}[theorem]{Proposition}
\newtheorem{lemma}[theorem]{Lemma}
\newtheorem{corollary}[theorem]{Corollary}
\theoremstyle{definition}
\newtheorem{definition}[theorem]{Definition}

\theoremstyle{remark}

\usepackage{multirow}

\usepackage[T1]{fontenc}
\usepackage{aecompl}
\usepackage{svg}

\makeatletter
\newcommand*{\addFileDependency}[1]{
\typeout{(#1)}
\@addtofilelist{#1}
\IfFileExists{#1}{}{\typeout{No file #1.}}
}\makeatother

\definecolor{darkred}{rgb}{0.95, 0.1, 0.1}
\newcommand{\yusu}[1] {{\textcolor{darkred}{{\sc [From YW:]} #1}}}

\newcommand{\denselist}{\itemsep 0pt\parsep=1pt\partopsep 0pt}
\newcommand{\Acal} {\mathcal{A}}

\newcommand{\mset}[1]{\{\!\{#1\}\!\}}

\usepackage[accepted]{aistats2024}
\usepackage[round]{natbib}
\renewcommand{\bibname}{References}
\renewcommand{\bibsection}{\subsubsection*{\bibname}}

\begin{document}

%
\runningtitle{Higher-Order Graph Transformers}

%
\runningauthor{Cai Zhou, Rose Yu, Yusu Wang}

\twocolumn[

\aistatstitle{On the Theoretical Expressive Power and the Design Space of \\Higher-Order Graph Transformers}

\aistatsauthor{Cai Zhou$\dagger$\footnotemark[1] \And Rose Yu$^*$ \And  Yusu Wang$^*$ }

\aistatsaddress{$\dagger$Tsinghua University \qquad  $^*$University of California, San Diego \\
{\texttt{zhouc20@mails.tsinghua.edu.cn}, \qquad  \texttt{\{roseyu,yusuwang\}@ucsd.edu}}
} ]

\begin{abstract}
  Graph transformers have recently received significant attention in graph learning, partly due to their ability to capture more global interaction via self-attention. Nevertheless, while higher-order graph neural networks have been reasonably well studied, the exploration of extending graph transformers to higher-order variants is just starting. Both theoretical understanding and empirical results are limited. 
  In this paper, we provide a systematic study of the theoretical expressive power of order-$k$ graph transformers and sparse variants. We first show that, an order-$k$ graph transformer without additional structural information is less expressive than the $k$-Weisfeiler Lehman ($k$-WL) test despite its high computational cost. 
  We then explore strategies to both sparsify and enhance the higher-order graph transformers, aiming to improve both their efficiency and expressiveness. Indeed, sparsification based on neighborhood information can enhance the expressive power, as it provides additional information about input graph structures. In particular, we show that a natural neighborhood-based sparse order-$k$ transformer model is not only computationally efficient, but also expressive -- as expressive as $k$-WL test. We further study several other sparse graph attention models that are computationally efficient and provide their expressiveness analysis. Finally, we provide experimental results to show the effectiveness of the different sparsification strategies. 
\end{abstract}

\section{INTRODUCTION}

Recent years have witnessed great success in using the Transformer architecture for various applications, e.g., in natural language processing~\citep{AttentionIsAll}, computer vision~\citep{Vit}, and more recently graph learning~\citep{GPS, Specformer}. However, applying transformers to graphs is fundamentally different from texts or images because the graph topology and structure information cannot be easily captured by standard transformers. Furthermore, while graph neural networks benefit from higher-order extensions known as Invariant Graph Networks (IGN)~\citep{IGN, ExpressivepowerkIGN}, higher-order transformers have not yet been well studied, whose theoretical benefits and empirical strengths are still largely unexplored. 

Specifically, on the theoretical front, analysis of the expressive power of graph transformers is currently limited, especially for the higher-order variants. \citet{ConnectionMPNNGT} show that a Message Passing Neural Network (MPNN) with a virtual node can approximate certain kernelized graph transformers (Performer \citep{Performer} and Linear Transformer \citep{LinearTransformerRNN}). 
\citet{TransformersDeepsetsGraphs} proposed a higher-order attention mechanism operating on $k$-tuples (instead of graph nodes), inspired by a generalization of linear equivariant layers \citep{IGN}. \citet{PureTransformerspowerful} introduced a high-order graph transformer model called TokenGT, and showed that order-$k$ TokenGT is as expressive as the so-called $k$-WL (order-$k$ Weisfeiler-Lehman \citep{CFIgraph}) test. However, their result requires that the input $k$-tuples be equipped with orthonormal vectors (of length $O(n)$) as node identifiers. Furthermore, it is not clear how to choose such node identifiers in a permutation-invariant manner. 

On the empirical front, $k$-transformers are computationally expensive, and the order-$k$ transformer of \citep{PureTransformerspowerful} takes time $O(n^{2k})$ for a graph with $n$ nodes. (Note that higher-order graph networks are also expensive; for example, the expressive order-$k$ Invariant Graph Network (i.e. $k$ -IGN) also takes $O(n^{2k})$ time to compute for each layer.) Hence, it is valuable to explore sparse versions for improved efficiency.

\paragraph{Our Work.}
In this paper, we provide a systematic study of the theoretical expressive power of order-$k$ graph transformers and their sparse variants. Our main contributions are as follows. See also the summary of some theoretical results in Table \ref{Table_summary_design_space}. 
\begin{itemize}\denselist 
    \item In \cref{SectionTheory}, after formally introducing a natural formulation of order-$k$ transformers $\Acal_k$, we show that without ``indices'' information of $k$-tuples, $\Acal_k$ is strictly less expressive than $k$-WL. But when augmented with the indices information, its expressive power is at least that of $k$-WL. 
    Note that however, each layer of $\Acal_k$ takes $O(n^{2k}d)$ time (see Table \ref{Table_summary_design_space}) where $d$ is the network width (latent dimension). Unfortunately, similar to \citep{PureTransformerspowerful}, the resulting model may not be invariant to the choices of the indices. 
    \item In \cref{SectionPractical}, we explore strategies to improve the efficiency of higher-order graph transformers while maintaining strong expressive power. In fact, sparsifying attention using the graph structure can enhance the expressive power. We propose several sparse high-order transformers in \cref{SectionPractical}, and analyze their expressiveness and time complexity. A particularly interesting one is what we call the \emph{neighbor-attention} mechanism, which, as shown in Table \ref{Table_summary_design_space}, significantly improves the time complexity (from $O(n^{2k})$ to $O(n^k)$ roughly speaking), while is as expressive as $k$-WL. Furthermore, note that this model doesn't use indices, and is permutation invariant. 
    In \cref{SectionSimplicialTransformer}, we also study simplicial complexes-based higher-order transformers. 
    \item In \cref{SectionExperiment} and \cref{SectionExperimentsAppendix}, we also provide experimental results to show the relative performances of these higher-order transformers (although mostly of order $2$), and compare them with SOTA methods in graph learning. 
\end{itemize}

\begin{table*}[t]
\caption{Summary of expressiveness and complexity of high-order transformers $\mathcal A_{k_1,k_2}$, their sparse variants, and simplicial transformer $\mathcal {AS}_{k_1:k_2}$. We also include the results of $k$-IGN of \citep{IGN} for reference. 
Here $n$ is the number of nodes; $\bar D$ is the average node degree; $d$ is the latent dimension of model (sometimes omitted in main text); $\mathcal S_k$ is the set of $k$-simplices; \textit{simplex neighbor} includes coboundaries, boundaries, upper and lower adjacent neighbors, while $\bar{ D_{\mathcal S}}$ is the average number of these extended neighbors. Details of the simplicial transformers $\mathcal AS_{*}$s are in \cref{SectionSimplicialTransformer}.}
\label{Table_summary_design_space}
\vskip -0.1in
\begin{center}
\begin{small}
\resizebox{2.\columnwidth}{!}{
\begin{tabular}{ccccccc}
\toprule
Base & \multirow{2}{*}{Type} & Computation & Sparse & \multirow{2}{*}{Enhancement} & \multirow{2}{*}{Complexity} & Expressive \\
model & & reduction & attention & & & power\\
\midrule
$k$-IGN~\citep{IGN} & - & - & - & - & $O(n^{2k}d)$  & $=k$-WL\\
\midrule
$\mathcal A_k$ & dense & - & - & - & $O(n^{2k}d)$  & $\prec k$-WL\\
$\mathcal A_k$ & dense & - & - & input indices & $O(n^{2k}d)$  & $\succeq k$-WL(with ReLU)\\
$\mathcal A_k$ & dense & kernelization & - & - & $O(n^{k}d^2)$  & $\prec k$-WL\\
$\mathcal A_k^{\mathsf{Ngbh}}$ & sparse & - & neighbor & - & $O(n^{k+1}kd)$  & $=k$-WL\\
$\mathcal A_k^{\mathsf{LN}}$ & sparse & - & local neighbor & - & $O(n^{k}k\bar D d)$  & $\succeq \delta\text{-}k$-LWL\\
$\mathcal A_k^{\mathsf{VT}}$ & sparse & - & virtual tuple & - & $O(n^{k} d)$  & $\simeq$ kernelized $\mathcal A_k$\\
\midrule
$\mathcal {AS}_k$ & dense & - & - & $\mathbf L_k$ & $O(|\mathcal S_k|^2d)$  & -\\
$\mathcal {AS}_{0:K}$ & dense & - & - & $\mathcal L_{0:K}$ & $O((\sum_{k=0}^K|\mathcal S_k|)^2d)$  & $\succeq $MPSN \\
$\mathcal {AS}_{0:K}^{\mathsf {SN}}$ & sparse & - & simplex neighbor & - & $O((\sum_{k=0}^K|\mathcal S_k|) \bar {D_{\mathcal S}}d)$  & $= $MPSN \\
$\mathcal {AS}_{0:K}^{\mathsf {VS}}$ & sparse & - & virtual simplex & - & $O((\sum_{k=0}^K|\mathcal S_k|) d)$  & $\simeq$ kernelized $\mathcal {AS}_{0:K}$\\

\bottomrule
\end{tabular}
}
\end{small}
\end{center}
\vskip -0.1in
\end{table*}

\section{PRELIMINARIES AND BACKGROUND}

In this section, we briefly review two notions: the order $k$-Weisfeiler-Lehman ($k$-WL) test, which is commonly used as a way to measure expressiveness of graph neural networks, as well as the $k$-IGN \citep{IGN}, which is the most expressive graph networks of a given order($k$-IGN is as expressive as $k$-WL).  

\textbf{$k$-WL test} is a procedure to test whether two input attributed graphs are potentially isomorphic or not. 
In the literature on graph networks, the $k$-WL procedure usually refers to the following iterative color-assignment scheme for an input graph $G$. The $k$-WL test will return that two graphs are ``not isomorphic'' if the collections of colors assigned at any iteration differ; otherwise, it will return ``potentially isomorphic". 

In particular, given a graph $G$, the $k$-WL procedure assigns colors to all $k$-tuples of $V(G)$ and iteratively updates them. The initial color $c_k^{0}(\mathbf v, G)$ of the tuple $\mathbf v\in V(G)^k$ is determined by the isomorphism type of the tuple $\mathbf v$~\citep{PPGN}. At the $t$-th iteration, the color updating scheme is
\begin{equation}\label{equation_kWL_color}
\begin{aligned}
    c_k^{t}(\mathbf v, G)=& \text{Hash}\Big(c_k^{t-1}(\mathbf v, G),\\ \big(\mset{&
    c_k^{t-1}(\psi_i(\mathbf v, u), G)|u\in V(G)}|i\in [k]\big)\Big)
\end{aligned}
\end{equation}
where $\psi_i(\mathbf v, u)$ means replacing the $i$-th element in $\mathbf v$ with $u$. 
The color of the entire graph is the multiset of all tuple colors,
\begin{equation}
    c^{t}_k(G)={\rm Hash}\big(\mset{c^{t}_k(\mathbf v, G)|\mathbf v\in V(G)^k}\big)
\end{equation}
As mentioned earlier, two graphs are considered ``potentially isomorphic" by $k$-WL if they have identical tuple color multisets for all iterations $t$s. It is known that $1$-WL is equivalent to $2$-WL in terms of differentiating graphs, while $k$-WL is strictly less powerful than ($k+1$)-WL for $k \ge 2$~\citep{CFIgraph}. 

\textbf{$k$-IGN} was introduced by \cite{IGN}. An order $k$ Invariant Graph Network ($k$-IGN) is defined as a function $F_k:\mathbb R^{n^k\times d_o}\rightarrow \mathbb R$ of the following form:
\begin{equation}
    F_k={\rm MLP} \ \circ \ L_{k\rightarrow 0} \ \circ \ L_{k\rightarrow k}^{(T)}  \ \circ \ \sigma  \ \circ \ \dots  \ \circ \ \sigma  \ \circ \ L_{k\rightarrow k}^{(1)},
\end{equation}
where each $L_{k\rightarrow k}^{(t)}$ is a (permutation) \emph{equivariant linear layer} $\mathbb R^{n^k\times d_{t-1}}\rightarrow \mathbb R^{n^k\times d_t}$, $L_{k\rightarrow 0}$ is an \emph{invariant linear layer} $\mathbb R^{n^k\times d_T}\rightarrow \mathbb R$, while $\sigma$ is nonlinear activation function such as ReLU. In particular, it turns out that there exist ${\rm bell}(k+l)$ number of basis tensors $\mathbf B^\mu$ and ${\rm bell}(l)$ number of basis tensors $\mathbf C^\lambda$, such that any equivariant linear layer $L_{k\rightarrow l}: \mathbb R^{n^k\times d}\rightarrow \mathbb R^{n^l\times d'}$ from a $k$-order tensor to a $l$-order tensor can be written as follows: for order $k$ input $\mX \in \mathbb R^{n^k\times d}$,
\begin{equation}\label{equation_kIGN}
    L_{k\rightarrow l}(\mX)_{\vi} =\sum_{\mu}\sum_{\vj} \mathbf B_{\vi,\vj}^\mu \mX_{\vj} w_\mu + \sum_\lambda \mathbf C_{\vi}^\lambda b_\lambda. 
\end{equation} 
Here $\vi\in [n]^l,\vj \in [n]^k$ are multi-indices, $w_\mu \in \mathbb R^{d\times d'}, b_\lambda \in \mathbb R^{d'}$ are weight and bias parameters, and $\mathbf B^\mu \in\mathbb R^{n^{l+k}}$.
Note that an invariant layer is simply a special case of an equivariant layer $L_{k\rightarrow l}$ when $l=0$. 
It is known that $k$-IGNs are as expressive as $k$-WL \citet{PPGN,ExpressivepowerkIGN,ExpressiveInvariantEqui}. Analogous to the construction of $k$-IGN, all graph networks including graph transformers have permutation equivariant intermediate layers, and permutation invariant final graph-level outputs.

\section{THEORETICAL ANALYSIS OF HIGH-ORDER TRANSFORMERS}\label{SectionTheory}

In this section, we first introduce a natural notion of higher-order transformers, which is slightly more general than the one in \cite{PureTransformerspowerful}. We show that a plain order-$k$ transformer without additional structural information is \textit{strictly less expressive} than $k$-WL. On the other hand, adding explicit tuple indices as part of the input tuple features leads to a $k$-transformer that is at least as powerful as $k$-WL. However, it may also break the permutation invariance (not invariant to the choice of the indices). The proofs and some additional results can be found in \cref{SubsecProofSec3}. 

\begin{definition}[Order $k_1,k_2$-Transformer Layer]\label{def_k1k2-transformer}
A \emph{(cross-attention) transformer layer} $\mathcal A_{k_1,k_2}:\mathbb R^{n^{k_1}\times d}\times \mathbb R^{n^{k_2}\times d}\rightarrow \mathbb R^{n^{k_1}\times d'}$ takes a query tensor $\mX\in\mathbb R^{n^{k_1}\times d}$ and a key tensor $\mY\in \mathbb R^{n^{k_2} \times d}$ as input, and is defined as
\begin{equation}\label{k1k2-s1s2transformer}
    \mathcal A_{k_1,k_2}(\mX, \mY)={\rm softmax}\Big(\underbrace{(\mX Q)}_{\in \mathbb R^{n^{k_1}\times d_k}} \underbrace{(\mY K)^\top }_{\in \mathbb R^{d_k\times n^{k_2}}}\Big) \underbrace{(\mY V)}_{\in \mathbb R^{n^{k_2}\times d'}}. 
\end{equation}
    Here $Q\in \mathbb R^{d\times d_k}$, $K\in \mathbb R^{d\times d_k}$, $V\in \mathbb R^{d\times d'}$ are the learnable weight matrices where $d_k$ is the latent dimension, and the softmax is performed row-wise on an $n^{k_1}\times n^{k_2}$ matrix. In practice $d, d_k, d'$ are usually of the same magnitude and are thus both denoted as $O(d)$ regarding complexity. We may also sometimes omit the $O(d)$ factor in the main text when we care more about the complexity w.r.t number of nodes $n$.
    
    A \textit{self-attention transformer layer} is when $\mX =\mY, \ k_1=k_2=k$, which we also refer to as \textit{(order) $k$-transformer layer}, and denoted by $\mathcal A_k$ for simplicity. 
\end{definition}

Throughout the paper, we mostly talk about self-attention transformers (which is standard in the literature). However, some experiments using cross-attention transformers are given in \cref{SectionExperimentsAppendix}. 


A $k$-transformer is composed of many layers of the above $k$-transformer layers. Note that the above definition is for single-head attention, but it is easy to extend it to multi-head attention analogous to standard transformers. Also, following the standard setting, we allow bias terms after multiplying $\mX$ with query, key and value weight matrices; although these are omitted in the above definition for clarity. Furthermore, residual connection and input/output MLPs are always allowed between these $k$-transformer layers.
Note that Definition \ref{def_k1k2-transformer} is generic and not limited to graph data. For graph input, similar to the $k$-WL procedure, the feature initialization of $k$-tuples can be based on the isomorphism types of the tuples~\citep{PPGN}. 

\paragraph{Theoretical Expressive Power of Order-$k$ Transformer.}

Next, we investigate the expressive power of order-$k$ transformer in terms of the $k$-WL hierarchy. A subtle issue here is the use of indices: $k$-WL test assumes that the graph nodes are indexed, and a $k$-tuple is explicitly {\bf indexed} by $k$ indices of those nodes in this tuple. This index is only used to compute the ``neighbors'' of a specific $k$-tuple (see the use of $\psi_i(\bm v, u)$ in Eqn (\ref{equation_kWL_color}). The entire procedure is still \emph{independent} to the indexing. 
Such structural information unfortunately is lost in a $k$-transformer layer, where the self-attention mechanism cannot differentiate such ``neighbors''. It is therefore not surprising that we have the following negative result (proof in \cref{subsubsec_proof_kWL_kFWL}). 


\begin{theorem}\label{Theorem_Akk<kWL}
    Without taking tuple indices as inputs, $\mathcal A_k$ is strictly less expressive than $k$-WL.
\end{theorem}
The limited expressive power of $\mathcal A_k$ is not desirable: It suffers from high computational cost -- indeed, the complexity for one $\mathcal A_k$ layer is the same as $k$-IGN, which is $O(n^{2k})$. Furthermore, it is less expressive than $k$-IGN (which is as expressive as $k$-WL).

A natural step forward is to make $\mathcal A_k$ more expressive through structural enhancements.
Given the discussion of the use of ``indices'' in $k$-WL earlier, a natural strategy is to bring in structural information by augmenting the inputs with tuple indices. Indeed, as Theorem \ref{theorem_Akk=kWL} below shows, this enhancement improves the expressiveness to the same as $k$-WL. Here we take the $k$-dimensional tuple indices as the model inputs, where each index $\vi \in [n]^k$ is the multi-index of the same definition as in $k$-IGN~\citep{IGN}.

%
\begin{theorem}\label{theorem_Akk=kWL}
    For inputs $\mX \in \mathbb R^{n^k\times (d+k)}$ where each element is a concatenation of a $d$-dimensional tuple feature and $k$-dimensional index of the tuple, one layer of $\mathcal A_k$ with latent dimension $O(k)$ and $k$ heads augmented with input MLPs, residual connection feed-forward layers can approximate one $k$-WL iteration arbitrarily well. If the softmax function is replaced by element-wise ReLU activation, $\mathcal A_k$ can exactly simulate $k$-WL.
\end{theorem}
Here we provide some brief comments on \cref{theorem_Akk=kWL}. It is interesting to see that $\mathcal A_k$ with ReLU activation augmented by tuple indices can simulate $k$ -WL and thus is permutation invariant. Unfortunately, unlike $k$-IGN, the resulting model is not guaranteed to be permutation invariant to choices of tuple indices - for example, the output of the softmax version of $\mathcal A_k$ is not permutation invariant to the tuple indices. Hence the above theorem is only of theoretical interest. Note that the same issue also applies to TokenGT~\citep{PureTransformerspowerful} -- in fact, they assume that each node has a distinct orthonormal vector as a ``node identifier'', which is an even stronger requirement than using just indices. Each node identifier needs a size of $O(n)$, thus TokenGT requires more time complexity ($O(n^{2k+1})$) to achieve $k$-WL expressiveness than our construction ($O(n^{2k})$). 
Our proof is also different from that of \cite{PureTransformerspowerful}. Instead of approximating the equivalence class basis in $k$-IGN as they do, we directly simulate $k$-WL. See \cref{subsubsec_proof_kWL_kFWL} for the proof of the above theorem, as well as more discussions regarding our advantages over the result of \cite{PureTransformerspowerful}.

\section{EFFICIENT AND EXPRESSIVE HIGH-ORDER GRAPH TRANSFORMERS}\label{SectionPractical}

The plain order-$k$ graph transformer model $\mathcal A_k$ suffers from  $O(n^{2k})$ time complexity, limited expressive power (\cref{Theorem_Akk<kWL}) or broken permutation invariance (\cref{theorem_Akk=kWL}). In this section, we explore sparse high-order graph transformers to address these issues. In particular, in Section \ref{subsec:kernelization}, we study a more general kernelization strategy to sparsify self-attentions. The more interesting exploration is presented in Section \ref{subsec:neighborhoodbased}, where we inject graph structure inductive biases to improve efficiency, while maintaining or even improving the expressive power of the resulting model. 
In Section \ref{subsec:sampling}, we also discuss how to reduce the computational complexity by using a reduced set of $k$-tuples with most details in \cref{SectionSimplicialTransformer}. The proofs of all the theorems in this section are in \cref{SubsecProofSec4}.

Contrary to the explicit use of indexing as in Theorem \ref{theorem_Akk=kWL} (as well as the use of orthonormal node identifiers for TokenGT \cite{PureTransformerspowerful}), all the models in this section do not require \textit{explicit indices as part of input features}, and the resulting attention layer is \textit{permutation equivariant}. 




\begin{figure*}[t]
\begin{center}
\centerline{\includegraphics[width=2.0\columnwidth]{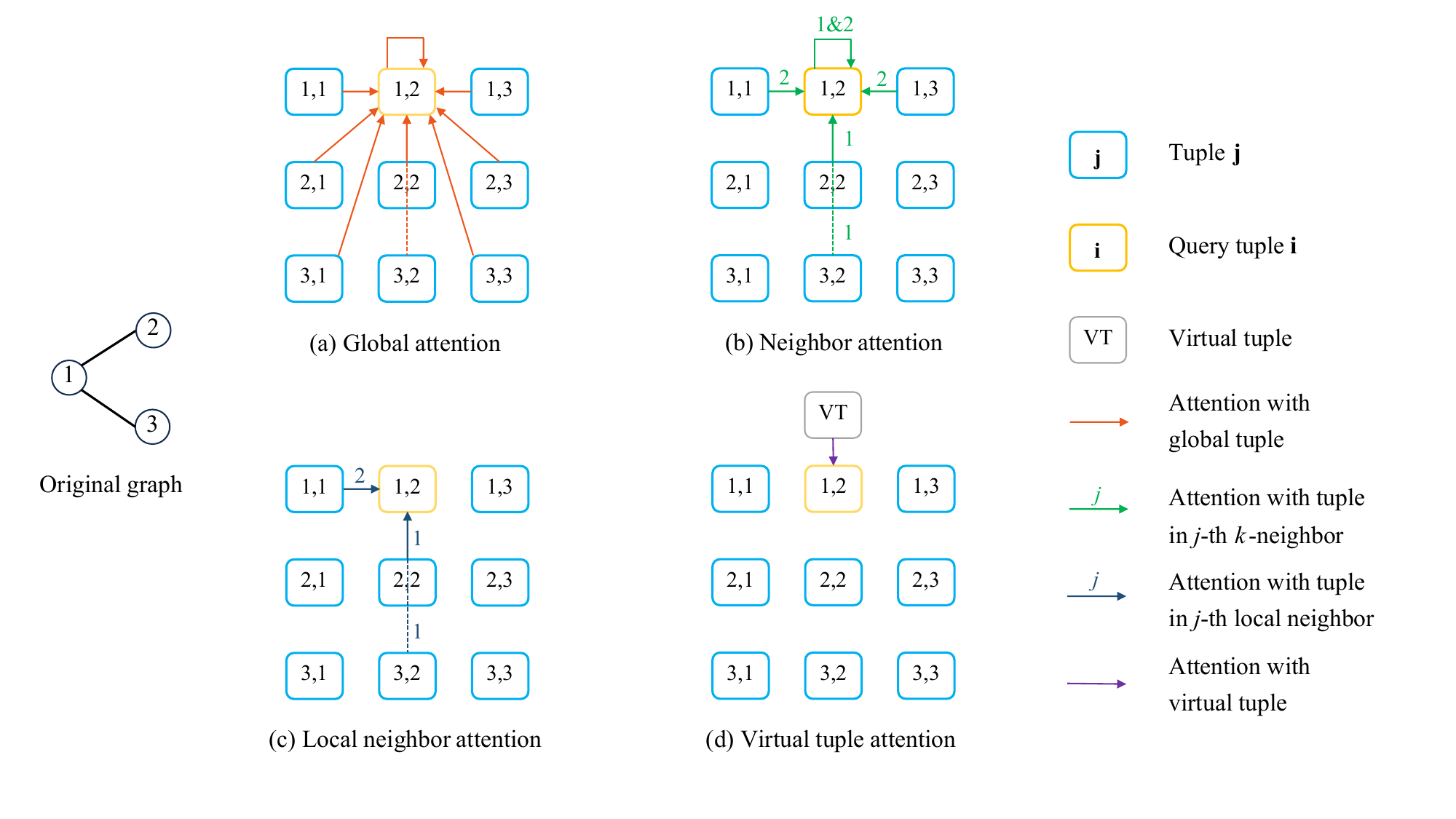}}
\vskip -0.15in
\caption{Variants of $k$-th order self-attention, i.e. $\mathcal A_k$ and its sparse forms. In the figure order $k=2$, number of nodes $n=3$. For simplicity, we only show the attention of query token $\vi=(1,2)$, and all $n^k$ (real) tuples are calculated with the same rule. The dashed lines are only for aesthetic illustration. (a) Global attention (plain $\mathcal A_k$), the query token computes attention with all $n^k$ tuples. (b) Neighbor attention, the query token computes attention with its $k$-neighbors; $k$-neighbor is of the same definition as in $k$-WL. (c) Local neighbor attention, where the query token computes attention with only its local neighbors; local neighbor is of the same definition as in \citep{WLgoSparse}. (d) Virtual tuple attention, the query token only computes attention with the virtual tuples (we only display one for simplicity), while each virtual tuple computes attention with all other real tuples.}
\label{figure_attention_illustration}
\end{center}
\vskip -0.15in
\end{figure*}

\subsection{Kernelized Attention}\label{subsec:kernelization}


Kernelized attention has been commonly used for (standard) transformers~\citep{Performer, LinearTransformerRNN}. In this subsection, we generalize kernelized attention to higher-order transformers, reducing the complexity from $O(n^{2k}d)$ to $O(n^k d^2)$, where $d$ is the feature dimension. We also provide a theoretical analysis of such generalization. 

Such a generalization is rather straightforward. This is because the order-$k$ self-attention layer essentially has the same structure as a standard self-attention layer, where the only difference is that the input tokens are now those $k$-tuples (instead of nodes in the standard transformer). 
For completeness, we include the formulation here. 
Specifically, a single head order-$k$ self-attention layer can be re-formulated as
\begin{equation}
\begin{aligned}
    \mX_{i}&=\sum_{j=1}^{n^k} \frac{\kappa(\mX_i Q, \mX_j K)}{\sum_{l=1}^{n^k} \kappa(\mX_i Q, \mX_l K)}\cdot (\mX_j V)\\
\end{aligned}
\end{equation}
where $\kappa:\mathbb R^d\times \mathbb R^d \rightarrow \mathbb R$ is the softmax kernel $\kappa(\mathbf x,\mathbf y)=\exp(\mathbf x^\top \mathbf y)$. The kernel trick approximates the softmax via
$\kappa(\mathbf x,\mathbf y)=<\Phi(\mathbf x),\Phi(\mathbf y)>\approx \phi(\mathbf x)^T\phi(\mathbf y)$
where the first equation is by Mercer's theorem and the latter is a low-dimensional approximation with transformation $\phi:\mathbb R^d\rightarrow \mathbb R^m$. In Performer \citep{Performer}, $\phi(\mathbf x)=\frac{\exp \big(-\frac{||X||_2^2}{2}\big)}{\sqrt m}[\exp(\mathbf w_1^\top \mathbf x),\dots, \exp(\mathbf w_m^\top\mathbf x)]$ where $\mathbf w_k\sim \mathcal N(0,I_d)$. In Linear Transformer \citep{LinearTransformerRNN}, $\phi(\mathbf x)=elu(\mathbf x)+1$. Then we can further rewrite the attention as
\begin{equation}
    \mX_{i}= \frac{\Big(\phi(\mX_i Q)^\top \sum_{j=1}^{n^k}\big( \phi(\mX_j K) \otimes (\mX_j V)\big) \Big)^\top}{\phi(\mX_i Q)^\top \sum_{l=1}^{n^k} \phi(\mX_l K)}
\end{equation}
where $\otimes$ is the outer product. Note that the two summations $\sum_{j=1}^{n^k}\big( \phi(\mX_j K) \otimes (\mX_j V)\big)\in\mathbb R^{md'}$ and $\sum_{l=1}^{n^k} \phi(\mX_l K)$ are shared for all query tokens, while the former has the bottleneck complexity $O(n^kmd')$. Again, since $m,d'$ are comparable to $d$ in scale , the total time complexity of kernelized $\mathcal A_k$ is $O(n^kd^2)$. 
The following results show whether $k$-IGN and kernelized $\mathcal A_k$ can approximate each other.
\begin{theorem}\label{Theorem_kIGN>linearA_kk11}
    $k$-IGN can approximate kernelized $\mathcal A_k$ with Linear Transformer or Performer architectures arbitrarily well. Kernelized $\mathcal A_k$ is strictly less powerful than $k$-IGN.
\end{theorem}



\subsection{Sparse Attention via Neighbor, Local Neighbor and Virtual Tuple} 
\label{subsec:neighborhoodbased}

The aforementioned kernel tricks reduce model complexity from the perspective of computation and are applicable to general higher-order transformer. However, their attention mechanisms are still dense, in the sense that every query token aggregates information from all other tokens. Additionally, we haven't incorporated any inductive biases of graph and topology. 

In this section, we aim to design sparse attention mechanisms that have the same or even stronger expressive power compared with full attention $\mathcal A_k$. This could be possible, because when each tuple only computes attention with local tuples (generalized 'neighbors' defined according to certain rules), the model becomes not only sparse in computation, but also aware of part of structural information, which could be regarded as some sort of enhancement. 

In particular, we explore three types of sparsification strategies (see \cref{figure_attention_illustration}), and study both their computational complexity and expressiveness. 
Let $[n]$ denote the set of integers from $1$ to $n$. In all the constructions below, we will index all graph nodes (arbitrarily) from $1$ to $n$, and index each $k$-tuple by $\bm i = (i_1, \ldots, i_k) \in [n]^k$. These indices will only be used in the selection of ``neighbors'', and the resulting attention layer will be permutation equivariant/invariant.


\paragraph{Neighbor Attention.}

The first variant is called \textit{neighbor attention} mechanism. Similar to the $k$-WL procedure (recall Eqn (\ref{equation_kWL_color})), given a $k$-tuple $\bm i\in [n]^k$, its $j$-th $k$-neighbor is the $n$ number of $k$-tuples in the set $\mset{\psi_j(\bm i, u) \mid u\in [n]}$, where $\psi_j(\vi, u)$ means replacing the $j$-th element in $\vi$ with $u$.
In the (multi-head) $k$-order neighbor attention, denoted as $\mathcal A_k^{\mathsf{Ngbh}}$, each query tuple $\vi$ only computes its attention with its $k$ number of $k$-neighbors (each consists of $n$ number of $k$-tuples) in each head respectively and concatenates the results. Particularly, when the model has $k$ heads, 
\begin{equation}\label{equation_neighbor_attention}
\begin{aligned}
    \Big(\mathcal A_k^{\mathsf{Ngbh}}(\mX, \mX)\Big)_{\vi}={\rm Concat}\bigg[{\rm softmax} \Big((x_\vi Q^j)^\top  \\ \big(x_{[\psi_j(\vi, u)|u\in[n]]} K^j\big)\Big) \big(x_{[\psi_j(\vi, u)|u\in[n]]} V^j \big| j\in[k]\big) \bigg]
\end{aligned}
\end{equation}
where 
$x_{[\psi_j(\vi, u)|u\in[n]]}\in \mathbb R^{n\times d},\ j\in[k]$ is just the feature of $j$-th neighbor of $\vi$, and $Q^j\in\mathbb R^{d\times d_k}, K^j\in \mathbb R^{d\times d_k}, V^j \in \mathbb R^{d\times d'}(j\in[k])$ 
are weight matrices (parameters) of the $j$-th head. One may easily observe that this neighbor attention $\mathcal A_k^{\mathsf{Ngbh}}$ resembles $k$-WL algorithm, as they both update their representations using its $k$-neighbors. 
Now we show (proof in \cref{SubsubsecProofSparse}) that neighbor attention $\mathcal A_k^{\mathsf{Ngbh}}$ is as powerful as $k$-WL, making it more expressive than $\mathcal A_k$ while enjoying much lower complexity. 
\begin{theorem}\label{TheoremNeighborAttn>=kWL}
    Neighbor attention $\mathcal A_k^{\mathsf{Ngbh}}$ with residual connection, output MLPs, and $k$ heads is as powerful as $k$-WL. Each such layer has $O(n^{k+1}kd)$ time complexity.
\end{theorem}
Interestingly, while $\mathcal A_k^{\mathsf{Ngbh}}$ may assign different weights to the $n$ tuples inside each $k$-neighbor via attentions, it has the same theoretical expressive power as $k$-IGN. But neighbor attention is much more efficient than $k$-IGN ($O(n^{k+1}kd)$ vs $O(n^{2k}d)$). Furthermore, we think that this more flexible attention might make neighbor attention better than $k$-IGN in real-world tasks, even regardless of their time complexity. 


\paragraph{Local Neighbor Attention.}

\citet{WLgoSparse} proposed a family of $\delta\text{-}k$-dimensional WL algorithms, which is strictly more powerful than $k$-WL. Denote by $\vi_j$ the $j$ th element in tuple $\vi$ and $N(i)$ the neighbors of node $i$. Compared with the $k$-neighbor in $k$-WL, $\delta\text{-}k$-WL augments each $j$-th $k$-neighbor with the connectivity between the node being replaced $\vi_j$ and the $n$ nodes replacing it. Formally, $\delta\text{-}k$-WL computes
\begin{equation}\label{equation_delta_kWL_color}
\begin{aligned}
    &c_{\delta\text{-}k}^{t}(\vi, G)=\text{Hash}\bigg(c_{\delta\text{-}k}^{t-1}(\vi, G),\\ \Big(\mset{&
    \big(c_{\delta\text{-}k}^{t-1}(\psi_j(\vi, u), G), {\rm adj}(\vi_j, u)\big)|u\in V(G)}\big|j\in [k]\Big)\bigg)
\end{aligned}
\end{equation}
where ${\rm adj}(\vi_j, u)=\mathds 1(u\in N(\vi_j))$ is a boolean variable indicating whether node $\vi_j$ and node $u$ are connected. Using the idea of $\delta\text{-}k$-WL, we can make $\mathcal A_{k}^{\mathsf {Ngbh}}$ potentially more expressive than $k$-WL by incorporating ${\rm adj}(\vi_j, u)$ via attention bias or attention reweighting, which we denote as $\mathcal A_{k}^{\mathsf {Ngbh+}}$, see \cref{SubsecProofSec4} for details. 

However, despite being potentially more expressive, $\mathcal A_{k}^{\mathsf {Ngbh+}}$ has the same $O(n^{k+1}kd)$ complexity as $\mathcal A_{k}^{\mathsf {Ngbh}}$. Now we continue to present a more sparse attention variant, namely \textit{local neighbor attention}. Recall \citet{WLgoSparse} proposed a local variant algorithm named $\delta\text{-}k$-{LWL}, which only updates each tuple with its local neighbors. Specifically, the $j$-th \textit{local neighbor} of a $k$-tuple $\vi$ is defined as
\begin{equation}\label{equation_local_neighbor}
    \mathcal N_j^{Local}(\vi):=\mset{\psi_j(\vi, v)|v\in N(\vi_j)}
\end{equation}
which is a multi-set consists of $D(\vi_j)$ $k$-tuples where $D(\vi_j)$ is the degree of node $\vi_j$ - each tuple replacing element $\vi_j$ with its neighboring nodes $N(\vi_j)$ (instead of all $n$ possible nodes as in $k$-WL). More detailed descriptions are given in \cref{SubsubsecProofSparse}. 

Inspired by \cite{WLgoSparse}, we propose the ($k$-head) order-$k$ \textit{local neighbor attention} $\mathcal A_k^{\mathsf{LN}}$,
\begin{equation}
\begin{aligned}
    \Big(\mathcal A_k^{\mathsf{LN}}(\mX, \mX)\Big)_{\vi}={\rm Concat}\bigg[{\rm softmax} \Big((x_\vi Q^j)^\top\\ \big(x_{[\psi_j(\vi, u)|u\in N(\vi_j)]} K^j\big) \Big) \big(x_{[\psi_j(\vi, u)|u\in N(\vi_j)]} V^j \big| j\in[k]\big) \bigg]
\end{aligned}
\end{equation}
where $Q^j, K^j, V^j (j\in[k])$ are weight parameters of the $j$-th head. In other words, $\mathcal A_k^{\mathsf{LN}}$ updates each query tuple according to its local neighbors defined in \cref{equation_local_neighbor}. 
For example, it is easy to see that the attention in \cref{figure_attention_illustration} (c) is sparser than \cref{figure_attention_illustration} (b); e.g., tuple $(1,3)$ is in the $2$-th $2$-neighbor of query tuple $(1,2)$, but not in the $2$-th local neighbor of $(1,2)$ since nodes $2$ and $3$ are disconnected. 


\begin{theorem}\label{Theorem_LocalNA>=deltakLWL}
    Local neighbor attention $\mathcal A_k^{\mathsf{LN}}$ with residual connection, output MLPs, and $k$ heads is at least as powerful as $\delta\text{-}k$-LWL. Each such layer has $O(n^kk\bar D d)$ time complexity, where $\bar D$ is the average node degree.
\end{theorem}

\paragraph{Virtual Tuple Attention.}

We now present the last model of sparse attention, \textit{ virtual tuple attention}. Virtual node is a widely adopted heuristic technique for message-passing neural networks (MPNNs), and more recently graph transformers \citep{Exphormer}. The approximation power and expressive power of MPNN + virtual node have been studied in \cite{ConnectionMPNNGT}.

Similar to the idea of the virtual node, we introduce a virtual tuple and propose the virtual tuple attention $\mathcal A_k^{\mathsf{VT}}$, where the virtual tuple computes attention with all other real tuples. Each real tuple only computes attention with the virtual tuple (as there is only one key in this case, the softmax always outputs $1$, thus similar to message passing). The time complexity of virtual tuple attention is $O(n^kd)$, since each of the $n^k$ real $k$-tuples only needs to compute attention with the virtual tuple. Note that in practice, we can use multiple virtual tuples to capture more complex patterns. 

Mathematically, denote the feature of the virtual tuple as $x'$, the input is augmented to $\mX'=[\mX, x']\in\mathbb R^{(n^k+1)\times d}$, then $\mathcal A_k^{\mathsf{VT}}$ is calculated as
\begin{align}
    \mathcal A_k^{\mathsf{VT}}(\mX', \mX')_{n^k+1}&= {\rm softmax}\Big((x'Q^1)^\top (\mX K^1)\Big) \mX V^1   \\
    \mathcal A_k^{\mathsf{VT}}(\mX', \mX')_{\vi}&= x' V^{2}
\end{align}
Note that as $\mathcal A_k$ does not include tuple indices nor equivalence class basis, its properties are the same as first-order transformer with $n^k$ one-dimensional input tokens. Therefore, when allowed to be augmented with input and output MLPs, virtual tuple attention can be regarded as a natural extension of ``simplified MPNN + virtual node''~\citep{ConnectionMPNNGT}, and analysis of that ``simplified MPNN + virtual node'' can be applied directly to $\mathcal A_k^{\mathsf{VT}}$. 
See \cref{SubsecProofSec4} for more precise descriptions and details. Following the results of \citep{ConnectionMPNNGT}, we thus obtain: 

\begin{proposition}\label{Proposition_virtual_tuple=Linear_A_kk11}
    $O(1)$ depth and $O(1)$ width virtual tuple attention $\mathcal A_k^{\mathsf{VT}}$ can approximate kernelized $\mathcal A_k$ with Performer or Linear-Transformer architecture arbitrarily well. 
\end{proposition}
Analogously, as $O(1)$ depth and $O(n^d)$ width MPNN + VN can simulate full standard transformer~\citep{ConnectionMPNNGT}, the $O(1)$ depth and $O(n^{kd})$ width virtual tuple attention can thus simulate full $\mathcal A_k$ transformer. 

\subsection{Reducing Input $k$-tuples}
\label{subsec:sampling}

The three sparse attention models in the previous section require $O(n^k)$ complexity which is dictated by the number of input tokens (i.e. the number of $k$-tuples). 
An orthogonal direction to reduce computational complexity is to reduce the number of query tokens. 

\paragraph{Simplicial Attention.} 
Instead of all $k$-tuples, one can select a subset of $k$-tuples according to certain systematic rules.
Simplicial complexes provide a language to model such choices. For example, in applied and computational topology, one approach is to view an input graph as the $1$-skeleton of a hidden space, and there have been various works to construct a simplicial complex from the graph that can reflect low or high-dimensional topological features of this hidden space \citep{DW22}. 
Note that a $p$-dimensional simplex is intuitively spanned by $p+1$ number of vertices. As an example, a simple way to construct a $p$-simplicial is to include all those ($p+1$)-tuples of graph nodes that form a clique in the input graph. (The resulting simplicial complex is the so-called flag complex, or clique complex.) In general, the number of ($k-1$)-simplices is much smaller than all possible $k$-tuples of graph nodes. 
In \cref{SectionSimplicialTransformer}, we propose \textbf{simplicial attention variants} and analyze their theoretical properties. To distinguish from the tuple-based transformers in our main text, we use $\mathcal{AS}$ to denote simplicial transformers both in Table \ref{Table_summary_design_space} and in experimental results.

\paragraph{Random Sampling.} Finally, we remark that one can also use a random subset of $k$-tuples (either uniformly, or w.r.t. some probabilistic distribution depending on input graph structure). We provide some empirical results of sampling connected $3$-tuples in \cref{SectionExperimentsAppendix}. It will be an interesting future direction to explore how to obtain theoretical guarantees in expressiveness or approximation power under certain sampling strategies. 

\section{EXPERIMENTS}\label{SectionExperiment}

We conduct experiments on both synthetic datasets and real-world datasets. Using our sparse attention techniques, we can now scale order-$2$ graph transformers to datasets containing relatively large graphs, such as the long-range graph benchmark (LRGB)~\citep{LRGB}. 
Our higher-order graph transformers and simplicial transformers show superior expressivity on synthetic datasets, and achieve competitive performance across several real-world datasets. 

Due to limited space, we leave most experimental results, implementation details, and in-depth analysis in \cref{SectionExperimentsAppendix}. As representative results, we report the performance of our different higher-order transformer variants on (1) synthetic datasets with structure awareness tasks~\citep{AttendingGT}, including detecting edges and distinguishing non-isomorphic circular skip links (CSL) graphs, see \cref{Table_edge_CSL}; (2) Zinc12k~\citep{Zinc}, a popular molecular property prediction dataset, see \cref{Table_zinc}. See \cref{SectionExperimentsAppendix} for more results, including substructure counting for synthetic tasks, as well as OGB~\citep{OGB} and LRGB~\citep{LRGB} benchmarks for real-world tasks. 
Our main goal is to verify the theoretical properties or scalability of different models, and provide empirical analysis on their pros and cons. 


\begin{table}[ht]
\caption{Structure awareness tasks on synthetic datasets. Shown is the mean $\pm$ std of $5$ runs with different random seeds. Perfect results are shown in \textbf{bold}. The experimental settings and baseline results are adopted from \citep{AttendingGT}.}
\label{Table_edge_CSL}
\begin{center}
\begin{small}
\resizebox{1.\columnwidth}{!}{
\begin{tabular}{lcc}
\toprule
\multirow{2}{*}{Model} & Edge detection & CSL\\
 & $2$-way Accuracy $\uparrow$ & $10$-way Accuracy $\uparrow$\\
\midrule
GIN & $98.11\pm 1.78$ & $10.00\pm0.00$\\
Graphormer & $97.67\pm 0.97$ & $90.00\pm0.00$\\
\midrule
Transformer & $55.84\pm 0.32$ & $10.00\pm 0.00$\\
Transformer+LapPE & $98.00\pm 1.03$ & $\textbf{100.00}\pm0.00$\\
Transformer+RWSE & $97.11\pm 1.73$ & $\textbf{100.00}\pm0.00$\\
\midrule
$\mathcal A_{2}$ & $\textbf{100.00}\pm0.00$ & $10.00\pm0.00$\\
$\mathcal A_{2}$+LapPE & $\textbf{100.00}\pm0.00$ & $\textbf{100.00}\pm0.00$\\
$\mathcal A_{2}$+RWSE & $\textbf{100.00}\pm0.00$ & $\textbf{100.00}\pm0.00$\\
\midrule
$\mathcal A_{2}^{\mathsf {LN}}$ & $\textbf{100.00}\pm 0.00$ & $\textbf{100.00}\pm 0.00$\\
$\mathcal A_{2}^{\mathsf {LN}}$+LapPE & $\textbf{100.00}\pm 0.00$ & $\textbf{100.00}\pm 0.00$\\
$\mathcal A_{2}^{\mathsf {LN}}$+RWSE & $\textbf{100.00}\pm 0.00$ & $\textbf{100.00}\pm 0.00$\\
\midrule
$\mathcal {AS}_{0:1}$+attn.bias & $97.80\pm 0.33$ & $36.67\pm 0.00$\\
$\mathcal {AS}_{0:1}$+attn.bias+LapPE & $98.47\pm 0.11$ & $\textbf{100.00}\pm 0.00$\\
$\mathcal {AS}_{0:1}$+attn.bias+RWSE & $98.10\pm 0.25$ & $\textbf{100.00}\pm 0.00$\\
\midrule
$\mathcal {AS}_{0:1}^{\mathsf {SN}}$ & $96.16\pm 0.37$ & $20.00\pm 0.00$ \\
$\mathcal {AS}_{0:1}^{\mathsf {SN}}$+LapPE & $99.98\pm 0.01$ & $\textbf{100.00}\pm 0.00$ \\
$\mathcal {AS}_{0:1}^{\mathsf {SN}}$+RWSE & $99.54\pm 0.08$ & $\textbf{100.00}\pm 0.00$ \\
\bottomrule
\end{tabular}
}
\end{small}
\end{center}
\end{table}

\paragraph{Results on Synthetic Datasets.} \cref{Table_edge_CSL} shows the results on some synthetic datasets to study the capability of various models to capture graph ``structures''. Edge detection is a binary classification task to predict whether there is an edge connecting two given nodes, while the CSL test is a ten-way classification task to distinguish non-isomorphic circular skip links (CSL) graphs, which requires awareness of distance. We adopt the same experimental settings and baseline choices as \citep{AttendingGT}. 
Edge detection is an easier task, while CSL requires expressivity more powerful than $1$-WL. 
Indeed, without position/structure encodings (PE/SE), we see that GIN, Transformer, and $\mathcal A_2$ fail in CSL task. Order-$2$ transformer with \textit{local neighbor attention} $\mathcal A_{2}^{\mathsf {LN}}$ achieves perfect results in both tasks without any PE/SE, indicating that it sometimes can distinguish non-isomorphism graphs that $1$-WL fails. This is consistent with our theory in Theorem \ref{Theorem_LocalNA>=deltakLWL}, as it is at least as expressive as $\delta$-2-LWL, which might differentiate graphs not distinguished by 1-WL (although it may also fail to distinguish graphs that can be differentiated by 1-WL). In general, these high-order models also benefit from PE and SE. 


\begin{table}[ht]
\caption{Results on ZINC~\citep{Zinc}. Shown is the mean $\pm$ std of 5 runs with different random seeds. Highlighted are the \textcolor{orange}{first}, \textcolor{teal}{second} and \textcolor{violet}{third} results. Experimental settings and baseline results are adopted from \citep{GPS}.}
\label{Table_zinc}
\begin{center}
\begin{small}
\begin{tabular}{lc}
\toprule
Model & Test MAE $\downarrow$ \\
\midrule
GCN~\citep{GCN} & $0.367\pm 0.011$ \\
GAT~\citep{GAT} & $0.384\pm 0.007$ \\
GatedGCN~\citep{ResidualGatedGCN} & $0.282\pm 0.015$\\
PNA~\citep{PNA} & $0.188\pm 0.004$\\
\midrule
CIN~\citep{CWNetworks} & $0.079\pm 0.006$\\
GIN-AK+~\citep{GNNAK} & $0.080\pm0.001$ \\
\midrule 
SAN~\citep{RethinkingGTLap} & $0.139\pm0.006$\\
Graphormer~\citep{Graphormer} & $0.122\pm 0.006$\\
EGT~\citep{EGT} & $0.108\pm 0.009$ \\
GPS~\citep{GPS} & \textcolor{teal}{$0.070\pm 0.004$} \\
\midrule
$\mathcal A_{2}\text{-}$Performer (ours) & $0.155\pm 0.008$ \\
$\mathcal A_{2}^{\mathsf {VT}}$ (ours) & $0.186\pm 0.009$ \\
$\mathcal A_{2}^{\mathsf {Ngbh}}$ (ours) & $0.081\pm 0.005$\\
$\mathcal A_{2}^{\mathsf {Ngbh+}}$ (ours) & $0.075\pm 0.005$\\
$\mathcal A_{2}^{\mathsf {LN}}$ (ours) & $0.086\pm 0.006$\\
$\mathcal A_{2}^{\mathsf {LN+VT}}$ (ours) & \textcolor{orange}{$0.069\pm 0.005$}\\
\midrule
$\mathcal {AS}_{0:1}^{\mathsf {SN}}$ (ours) & $0.080\pm 0.004$\\
$\mathcal {AS}_{0:1}^{\mathsf {SN+VS}}$ (ours) & \textcolor{violet}{$0.073\pm 0.004$}\\

\bottomrule
\end{tabular}
\end{small}
\end{center}
\end{table}

\paragraph{Results on Real-World Datasets.} Zinc-12k~\citep{Zinc} is a popular real-world dataset containing 12k molecules. The task is a graph-level molecular property (constrained solubility) regression. We adopt the experimental settings and SOTA baseline results from \citep{GPS}. The results in \cref{Table_zinc} reveal that: (1) $\mathcal A_2$-Performer and $\mathcal A_2^{\mathsf VT}$ achieve similar results, verifying \cref{Proposition_virtual_tuple=Linear_A_kk11}; however, there are no theoretical guarantees in their expressive power, which aligns with the fact that their results are not highly competitive. (2) Both $\mathcal A_{2}$ and simplicial transformer $\mathcal {AS}_{0:1}$ achieve satisfactory results ($<0.09$ test MAE) when using only (local) neighbors or simplex neighbors, revealing the importance of local structure awareness. (3) Although virtual tuple/simplex attention does not provably enhance expressivity, these empirical mechanisms can improve the practical performance of other sparse attentions: $\mathcal A_2^{\mathsf {LN+VT}}$ and $\mathcal {AS}_{0:1}^{\mathsf {SN+VS}}$ achieve $0.069$ and $0.073$ test MAE respectively (better or comparable with SOTA GPS results~\citep{GPS}, indicating the empirical strength of global information. 

In addition, we also consider positional encoding for our high-order transformers. The approaches to add positional encoding to our high-order transformers are detailed in \cref{SubsecImplementationAppendix}. We verify the effectiveness of positional encoding to our models on Zinc and Alchemy~\citep{alchemy}, both of which are graph-level regression datasets to predict molecule properties. In both datasets, we consider positional encodings (PE), including SignNet/BasisNet~\cite{SignBasisNet}, SPE~\cite{StabilityPE} and MAP~\cite{LaplacianCanonization}. Baseline results are adopted from these papers correspondingly. 

\begin{table}[t]
\caption{More results on ZINC with positional encodings (PE). Shown is the mean $\pm$ std of 4 runs with different random seeds.}
\label{Table_zinc_rebuttal}
\begin{center}
\begin{tabular}{lcc}
\toprule
Model & PE & Test MAE $\downarrow$ \\
\midrule
GatedGCN & SignNet(8) & $0.121\pm 0.005$\\
GatedGCN & SignNet(All) & $0.100\pm 0.007$\\
GatedGCN & MAP(8) & $0.120\pm 0.002$ \\
GINE & SignNet(16) & $0.147\pm 0.005$\\
GINE & SignNet(All) & $0.102\pm 0.002$\\
PNA & SignNet(8) & $0.105\pm 0.007$\\
PNA & SignNet(All) & $0.084\pm 0.006$\\
PNA & MAP(8) & $0.101 \pm 0.005$\\
GIN & SPE(8) & $0.074\pm0.001$\\
GIN & SPE(All) & \textcolor{teal}{$0.069\pm 0.004$}\\
\midrule
$\mathcal {AS}_{0:1}^{\mathsf {SN}}$ (ours) & SignNet(8) & $0.079\pm 0.005$\\
$\mathcal {AS}_{0:1}^{\mathsf {SN}}$ (ours) & SignNet(All) & $0.078\pm 0.006$\\
$\mathcal {AS}_{0:1}^{\mathsf {SN+VS}}$ (ours) & SPE(8) & \textcolor{violet}{$0.072\pm 0.005$}\\
$\mathcal {AS}_{0:1}^{\mathsf {SN+VS}}$ (ours) & SPE(All) & \textcolor{orange}{$0.067\pm 0.005$}\\

\bottomrule
\end{tabular}
\end{center}
\end{table}

\begin{table}[t]
\caption{Experiments on Alchemy~\cite{alchemy} with positional encodings (PE). Shown is the mean $\pm$ std of 4 runs with different random seeds.}
\label{Table_alchemy_rebuttal}
\begin{center}
\begin{tabular}{lcc}
\toprule
Model & PE & Test MAE $\downarrow$ \\
\midrule
GIN & None & $0.112\pm0.001$\\
GIN & SignNet(All) & $0.113\pm0.001$\\
GIN & BasisNet(All) & $0.110\pm0.001$\\
GIN & SPE(All) & $0.108\pm 0.001$\\
\midrule
$\mathcal A_{2}^{\mathsf {Ngbh+}}$ (ours) & SPE(All) & \textcolor{violet}{$0.094\pm0.001$}\\
$\mathcal A_{2}^{\mathsf {LN+VT}}$ (ours) & SPE(All) & \textcolor{teal}{$0.090\pm0.001$}\\
$\mathcal {AS}_{0:1}^{\mathsf {SN+VS}}$ (ours) & SPE(All) & \textcolor{orange}{$0.087\pm 0.001$}\\

\bottomrule
\end{tabular}
\end{center}
\end{table}

As shown in \cref{Table_zinc_rebuttal} and \cref{Table_alchemy_rebuttal}, on both Zinc and Alchemy datasets, our high-order transformers can benefit from positional encoding, and achieve much more competitive performance compared with simple message-passing based GNNs.

In summary, our models are the first second-order attention-only methods that achieve results comparable to SOTA methods on Zinc (and also other real-world datasets in \cref{SectionExperimentsAppendix}). From the perspective of model size and time complexity, our models are comparable to graphGPS \cite{GPS} (SOTA graph transformer): our models have similar numbers of parameters as GPS. In terms of running time: local neighbor, virtual tuple attention $\mathcal A_2^{\mathsf {LN+VT}}$ and simplicial transformers have similar running time as GPS ($\sim 20$s/epoch). Our slowest variant $\mathcal A_2^{\mathsf {Ngbh}}$ is only about two times slower than GPS. See \cref{SectionExperimentsAppendix} for full details.




\section{CONCLUDING REMARKS}

In this work, we theoretically analyze the expressive power of higher-order graph transformers and systematically explore the design space for efficient and expressive high-order graph transformers. We propose sparse high-order attention mechanisms that enjoy both low computational complexity and high expressivity. Moreover, the theoretical results and architectural designs can be naturally extended to simplicial transformers. We provide preliminary experimental results to verify the performance of our high-order graph transformers and their scalability.

For future work, we note that most existing theoretical analysis of graph neural networks and their higher-order analysis center around expressiveness w.r.t. $k$-WL hierarchies. However, expressive power is only one interesting factor and may not be able to capture other dimensions of the effectiveness of a graph model. For example, while a graph transformer does not have more expressive power than a standard MPNN, it is believed to be more effective in capturing certain long-range interactions. It will be interesting to explore other ways to measure the power of graph learning models and study the pros and cons of transformer vs. non-transformer models.

\subsubsection*{Acknowledgements}
This work was supported in part by the U.S. Army Research Office
under Army-ECASE award W911NF-07-R-0003-03, the U.S. Department Of Energy, Office of Science, IARPA HAYSTAC Program, CDC-RFA-FT-23-0069, NSF Grants \#2205093, \#2146343,\#2134274,  \#2112665 and \#2310411.

\balance
\bibliography{ref}

\section*{Checklist}



 \begin{enumerate}

 \item For all models and algorithms presented, check if you include:
 \begin{enumerate}
   \item A clear description of the mathematical setting, assumptions, algorithm, and/or model. [\textbf{Yes}/No/Not Applicable]
   \item An analysis of the properties and complexity (time, space, sample size) of any algorithm. [\textbf{Yes}/No/Not Applicable]
   \item (Optional) Anonymized source code, with specification of all dependencies, including external libraries. [Yes/No/Not Applicable]
 \end{enumerate}

 \item For any theoretical claim, check if you include:
 \begin{enumerate}
   \item Statements of the full set of assumptions of all theoretical results. [\textbf{Yes}/No/Not Applicable]
   \item Complete proofs of all theoretical results. [\textbf{Yes}/No/Not Applicable]
   \item Clear explanations of any assumptions. [\textbf{Yes}/No/Not Applicable]     
 \end{enumerate}

 \item For all figures and tables that present empirical results, check if you include:
 \begin{enumerate}
   \item The code, data, and instructions needed to reproduce the main experimental results (either in the supplemental material or as a URL). [\textbf{Yes}/No/Not Applicable]
   \item All the training details (e.g., data splits, hyperparameters, how they were chosen). [\textbf{Yes}/No/Not Applicable]
         \item A clear definition of the specific measure or statistics and error bars (e.g., with respect to the random seed after running experiments multiple times). [\textbf{Yes}/No/Not Applicable]
         \item A description of the computing infrastructure used. (e.g., type of GPUs, internal cluster, or cloud provider). [\textbf{Yes}/No/Not Applicable]
 \end{enumerate}

 \item If you are using existing assets (e.g., code, data, models) or curating/releasing new assets, check if you include:
 \begin{enumerate}
   \item Citations of the creator If your work uses existing assets. [\textbf{Yes}/No/Not Applicable]
   \item The license information of the assets, if applicable. [\textbf{Yes}/No/Not Applicable]
   \item New assets either in the supplemental material or as a URL, if applicable. [\textbf{Yes}/No/Not Applicable]
   \item Information about consent from data providers/curators. [\textbf{Yes}/No/Not Applicable]
   \item Discussion of sensible content if applicable, e.g., personally identifiable information or offensive content. [Yes/No/\textbf{Not Applicable}]
 \end{enumerate}

 \item If you used crowdsourcing or conducted research with human subjects, check if you include:
 \begin{enumerate}
   \item The full text of instructions given to participants and screenshots. [Yes/No/\textbf{Not Applicable}]
   \item Descriptions of potential participant risks, with links to Institutional Review Board (IRB) approvals if applicable. [Yes/No/\textbf{Not Applicable}]
   \item The estimated hourly wage paid to participants and the total amount spent on participant compensation. [Yes/No/\textbf{Not Applicable}]
 \end{enumerate}

 \end{enumerate}

\newpage
\appendix
\onecolumn

\section{RELATED WORK}\label{Section_relatedwork}

\paragraph{Sparse Transformers.}

Transformers have achieved great success in natural language processing and, more recently, computer vision tasks. Since Transformers were proposed, researchers have been making an effort to reduce the computational complexity through linear or sparse attention mechanisms. It is noticeable that most relevant works focus on transformers for NLP and are not specially designed for graph learning. These methods can be broadly classified into two categories: (1) computing internal sparse attention through architecture design, which reduces the complexity; (2) computing full attention while encouraging the models to learn sparse patterns. The former categories try to capture global information at a lower computation cost. In comparison, the latter categories do not reduce complexity; instead, they expect the learned patterns are sparse so that the extracted representations contain important features of the data, which also eases interpretability.  

The first class of internal sparse transformers have low complexity via designing their algorithms and mechanisms to compute attention scores. In the following complexity analysis, the length of sequence is denoted as $n$, while we ignore the hidden dimension $d$. Linformer \citep{Linformer} proposes to project $n$ keys and values into $k$ groups, so that the complexity is reduced from $O(n^2)$ to $O(n\times k)$. In practice, however, $\frac{n}{k}$ is usually set to a constant factor like $4$, so Linformer actually reduces the complexity by a constant scale, while revealing worse performance than full attention. Routing Transformer \citep{RoutingTransformer} only computes attention between queries and a small subset of their nearest keys found by $k$-means, reducing the complexity to $O(n^{1.5})$. Similarly, Reformer \citep{Reformer} also computes local attention, except that the nearest keys are sorted by locality sensitive hashing (LSH), which results in a complexity of $O(n\log n)$. Longformer~\citep{Longformer} introduces a localized sliding window based mask and a few global mask to reduce computation. Moreover, BigBird \citep{BigBird} combines three types of sparse attention mechanisms: random attention, sliding window attention, and attention with global tokens. Further, Performer \citep{Performer} uses the Fast Attention through Positive Orthogonal Random features (FAVOR+) mechanism to approximate Softmax kernels in attention, which reduces the computation to linear complexity. Linear transformer \citep{LinearTransformerRNN} also applies linearized attention with kernel tricks, which approximates another nonlinear activation kernel.

Now we turn to the second category of transformers, which compute full attention but encourage the models to discover sparse patterns. Since reducing complexity is our main purpose, we will not pay much attention on this type. As a representative model, \citet{AdaptivelySparseT} enables attention heads to have flexible, context-dependent sparse patterns, which is achieved by replacing softmax with $\alpha$-entmax: a differentiable generalization of softmax that allows low-scoring items to be assigned precisely zero weight. 

\paragraph{Graph Transformers.}

Graph transformers have recently achieved great attention. Since global self-attention across nodes is unable to reflect graph structures, there are a number of works exploring graph-specific architectural designs. For example, GAT \citep{GAT} restricts attention within local neighboring nodes, and Graphormer \citep{Graphormer} injects edge information into the attention mechanism via attention bias. More recently, some graph transformers \citep{GPS, Specformer, GraphIBwithoutMP, Exphormer} have achieved even greater success with State-of-the-Art performance in a variety of tasks. GPS \citep{GPS} combines the attention mechanism with message passing and positional/structure encodings, while Specformer \citep{Specformer} embeds spectral features into graph transformers. GRIT \citep{GraphIBwithoutMP} builds a transformer architecture without message passing, which consists of learned relative positional encodings initialized with random walk probabilities and a flexible attention mechanism that updates node and node-pair representations. However, while there are a great number of high-order graph networks like $k$-IGN~\citep{IGN}, high-order graph transformers have been rarely studied besides primary results in \citet{TransformersDeepsetsGraphs, PureTransformerspowerful}. People also have limited understanding towards the theoretical expressive power of graph transformers, especially for high-order cases.


\paragraph{Simplicial Networks.}\label{SubsecSimplicialNetworks}

Besides applying transformers directly on higher order $k$-tuples, there is another direction worth exploring, namely higher order simplicial complexes. Instead of all $k$-tuples, simplicial networks and simplicial transformers focus on $k$-order (directed) simplices that are usually much more sparse. For example, there are always $n^2$ $2$-tuples in a graph with $n$ vertex, while there are $m$ $2$-simplices in the simplicial complex extended from the vertex of the original (directed) graph, where $m\leq n^2$ is the number of edges in the graph. The internal sparsity nature of simplices indicates the advantages of attention computed within higher-order simplices over higher-order tuples.

Early simplicial networks are mainly based on message-passing or convolution mechanisms. \citet{HodgeNetGNedgedata} is the first to generalize GNN to $1$-simplicial (edge) data. \citet{MPSimplicialN, ConvolutionalLearningOnSimplicial} formulate the message-passing and convolutional networks on simplices and simplicial complexes. \citet{CWNetworks} further generalize the message-passing and convolutional networks to cellular complexes. In comparison, attention-based methods on simplicial complexes are less complete. \citet{SimplicialAttentionNet, SimplicialAttentionNeuralNet} still restrict their attention within the scope of upper and lower adjacent simplices. There is no current work that computes full attention across all $k$-simplices (and more generally simplices of different orders), nor has the expressive power of these simplicial attention networks been analyzed.

\paragraph{High Order Transformers.} \cite{PureTransformerspowerful} proposed a possible form of high-order transformers, which computes attention for input order-$k$ tensors $\mX\in\mathbb R^{n^k}$ as follows,
\begin{align}\label{kk-attention}
    {\rm Attn}(\mX)_{\vj}=\sum_{h=1}^H \sum_{\vi}\mathbf {\alpha}_{\vi,\vj}^h \mX_{\vi}V^hO^h\\
    \mathbf \alpha^h={\rm softmax}\bigg(\frac{\mX Q^h(\mX K^h)^\top}{\sqrt{d_k}}\bigg)
\end{align}
where the attention $\alpha^h\in \mathbb R^{n^k\times n^k}$, and $Q^h,K^h\in \mathbb R^{d\times d_k}, V^h\in \mathbb R^{d\times d_v},O^h\in \mathbb R^{d_v\times d}$ are learnable weights.

\citet{PureTransformerspowerful} proves that with augmented node and type identifiers as input, the above attention can approximate any equivalence class basis tensor $\mathbf B^\mu \in R^{n^{2k}}$ of linear equivariant layer $L_{k\rightarrow k}$~\citep{IGN} arbitrary well. Consequently, \citet{PureTransformerspowerful} concludes that a $k$-Transformer with node and type identifiers is at least as expressive as $k$-IGN, and hence $k$-WL. In our work, we incorporate this type of architecture into a larger family of transformers. We also show that a slightly modified transformer can simulate $k$-WL in a different way from approximating the equivalence class basis as described in \citet{PureTransformerspowerful}.


In the previously described transformer, the input is a high-order tensor $\mX\in \mathbb R^{n^k}$, while the parameters $Q^h,K^h, V^h, O^h$ are actually the same as the standard transformer. \citet{RepresentationalStrengthsTransformer} propose another form of ``high-order'' transformer, which has the same $\mX\in \mathbb R^n$ input tokens as in the standard transformer, but is parameterized with high-order weight.

Recall the notations for the \textit{column-wise Kronecker product}. For vectors $\vv^1 \in \mathbb R^{n_1}, \vv^2\in \mathbb R^{n_2}$, their \textit{Kronecker product} $\vv^1 \otimes \vv^2 \in \mathbb R^{n_1n_2}$ is defined as $(\vv^1 \otimes \vv^2)_{i_1-1}n_2+i_2=\vv_{i_1}^1\vv_{i_2}^2$. The \textit{column-wise Kronecker product} of matrices $\mA^1\in \mathbb R^{n_1\times m}$ and $\mA^2\in\mathbb R^{n_2\times m}$ is defined as
\begin{equation}
    \mA^1 \star \mA^2=\Big[ \mA_1^1|\dots|\mA_m^1\Big] \star \Big[ \mA_1^2|\dots|\mA_m^2\Big] = \Big[\mA_1^1\otimes \mA_1^2|\dots| \mA_m^1\otimes \mA_m^2\Big] \in\mathbb R^{n_1n_2\times m}
\end{equation}

The $s$-order self-attention in \cite{RepresentationalStrengthsTransformer} is then defined as follows (Definition 7 in \cite{RepresentationalStrengthsTransformer}). For order $s\geq 2$, input dimension $d$, output dimension $d'$, and weight matrices $Q, K^1,\dots, K_{s-1}\in \mathbb R^{d\times d_k}$, $V^1,\dots,V^{s-1}\in \mathbb R^{d\times d'}$, an \textit{$s$-order self-attention unit} is a function $f_{Q,K,V}:\mathbb R^{n\times d}\rightarrow \mathbb R^{n\times d'}$ defined as
\begin{equation}\label{1s-1attention}
    f_{Q,K,V}(\mX)={\rm softmax}\Big(\underbrace{\mX Q}_{\in \mathbb R^{n\times d_k}} \underbrace{\big((\mX K^1)\star \dots \star (\mX K^{s-1})\big)^\top \Big)}_{\in \mathbb R^{d_k\times n^{s-1}}} \underbrace{\big((\mX V^1)\star \dots\star (\mX V^{s-1})\big)}_{\in \mathbb R^{n^{s-1}\times d'}}
\end{equation}
\citet{RepresentationalStrengthsTransformer} gives some primary analysis on the representation strength of the above high-order transformer over standard transformer (referred as $2$-order in their paper, yet $1$-order in our framework). Since the transformers in \citep{RepresentationalStrengthsTransformer} are different from our formulation, we leave relevant discussions in \cref{SectionAppendixDiscussion}.

\section{PROOF AND ADDITIONAL RESULTS}\label{SectionProofAppendix}

\subsection{Proof and Additional Results for Section 3}\label{SubsecProofSec3}

\subsubsection{Notations and Preliminaries}\label{SubsubsecNotation}

To start with, we first recall the definitions of $k$-neighbors in $k$-WL and $k$-FWL. For $k$-WL, recall 

\begin{equation}\label{equation_k-neighbor}
    \mathcal N_j(\vi)=\bigg\{(i_1,\dots, i_{j-1},i',i_{j+1},\dots,i_k)\Big|i'\in [n]\bigg\} . 
\end{equation}

$\mathcal N_j(\vi), j\in[k]$ is the $j$-th neighbor of tuple $\vi$ in WL algorithm, which is a set of $n$ different $k$-tuples. Each tuple $\vi$ has $k$ such $k$-neighbors, and during the update stage of the $k$-WL algorithm, these $k$ neighbors are aggregated as an {\bf ordered set} of multisets:
\begin{equation}\label{equation_kWL}
    {\rm WL:} \mC_{\vi}^{t+1}={\rm hash}\Bigg(\mC_{\vi}^t, \bigg(\mset{\mC_{\vv}^t\big| \vv \in \mathcal N_j(\vi)}\Big|j\in [k] \bigg)\Bigg)
\end{equation}

Note that this expression is equivalent to \cref{equation_kWL_color} in the main text. There is another form of graph isomorphic algorithm called Folklore Weisfeiler-Lehman (FWL) test, and the difference between $k$-WL and $k$-FWL lies in their tuple color update process. Concretely, the neighbors for $k$-FWL~\citep{PPGN} are defined as

\begin{equation}\label{equation_kFneighbor}
    \mathcal N_j^F(\vi)=\bigg((j,i_2,\dots, i_k), (i_1,j,\dots,i_k), \dots,(i_1,\dots,i_{k-1},j)\bigg)
\end{equation}

$\mathcal N_j^F(\vi),j\in[n]$ is the $j$-th neighborhood of tuple $\vi$ used by FWL, which is an ordered set of $k$ different $k$-tuples. Each tuple has $n$ such $k$-neighbors, while they are aggregated as a {\bf multi-set} of ordered sets in FWL update rule:
\begin{equation}\label{equation_kFWL}
    {\rm FWL:} \mC_{\vi}^{t+1}={\rm hash}\Bigg(\mC_{\vi}^t, \bigg\{\!\bigg\{\Big(\mC_{\vv}^t\big| \vv \in \mathcal N_j^F(\vi) \Big)\Big|j\in [n] \bigg\}\!\bigg\}\Bigg)
\end{equation}

We clarify the following technical details following \citep{PPGN}, which are applied in most of the relevant papers on expressive power. We adopt these techniques and conclusions throughout our proof. The techniques explain how neural networks (including transformer variants) can implement WL-like algorithms. 

\begin{itemize}
    \item (Remark 1) Color representation. We use tensors to represent colors. Concretely, the color of the $k$-tuple $\vi\in [n]^k$ is represented by a vector $\vx\in \mathbb R^d$ for some latent dimension $d$, and the entire color set of all $k$-tuples is $\mX\in\mathbb R^{n^k\times d}$.
    \item (Remark 2) Multiset representation. Unlike tuples, the multiset should be invariant to the order of nodes, i.e., $g\cdot \mX$ should be the same for all permutations $g\in S_n$, where $S_n$ is the symmetric group. As pointed out by \citep{PPGN}, \textit{Power-sum Multi-symmetric Polynomials} (PMP)~\citep{PMP} are $S_n$ invariant, thus can be used to encode multisets. As shown in \citep{PPGN}, PMP generates a unique representation of each multiset, which enables us to represent multisets as tensors.
    \item (Remark 3) Hash function. Suppose that there are two color tensors $\mC \in\mathbb R^{n^k\times a},\mC'\in\mathbb R^{n^k\times b}$, the hash function in \cref{equation_kWL} and \cref{equation_kFWL} can be implemented through a simple concatenation, that is, the new tensor representation for each $k$-tuple $\vi$ of color pair $(\mC_i, \mC_i')$ is simply $(\mC, \mC')\in\mathbb R^{n^k\times(a+b)}$.
    \item (Remark 4) Input and initialization. We adopt the same input and initialization of $k$-tuples as in Appendix C.1 of \citep{PPGN}, which can faithfully represent the isomorphism type of each $k$-tuple. That is, two tuples have identical initialization if and only if they have the same isomorphism types.
    \item (Remark 5) Update step of WL-like algorithm. In every update step of WL-like algorithms, when we use PMP to represent tuple colors, the color representation of a new multiset $\mset{\vj}$ consisting of tuples $\vj$ with known colors $\mB_\vj$ can be implemented by the summation of polynomial transform over $\mB_\vj$, i.e. $\mC(\mset{\vj})=\sum_\vj \tau(\mB_\vj)$, where $\tau$ is a specifically designed polynomial function that is injective, see \citep{PPGN} for more details. In practical networks, the polynomial function can be replaced by an MLP using the universal approximation power of MLPs~\citep{ApproximationMLP}. When the approximation power is sufficiently small, the injective properties can still be preserved to achieve the upper bound of the theoretical expressive power~\citep{PPGN}. Most relevant papers adopt this assumption, and we follow this as well in this paper. However, we emphasize that practical networks may not achieve their full expressive power when they are parameterized with learnable MLPs and embedding layers, other than strictly using PMP and polynomial functions. As for new tuples consisting of known tuples, as discussed above, the new color representation is simply the concatenation of known tuple colors.
    \item (Remark 6) Histogram computation. The final step to determine the graph isomorphism is the histogram of tuple colors $H(\mB)$. Suppose the number of colors is $b$ (which is finite for finite graphs); we apply a tuple-wise MLP $m:\mathbb R^d\rightarrow \mathbb R^b$ mapping each tuple color to a one-hot vector in $\mathbb R^b$. Summing over the one-hot vectors via a summing-invariant operator $h:\mathbb R^{n^k\times b}\rightarrow \mathbb R^b$, we obtain an injective and permutation-invariant representation of the color histogram: $H(\mB)=h(m(\mB))$.
\end{itemize}

In addition, while considering the approximation power of neural networks, we always assume that the inputs and network weights are compact. Specifically, for input $\mX$, we have $\forall \vi\in [n]^k, ||\mX_\vi||< C_1$, and for weight matrices of the transformer, we have $||Q||<C_2, ||K||<C_2, ||V||<C_2$. This is a mild and proper condition on feature space and parameter space which generally holds for practical networks.

\subsubsection{Approximating Power and Theoretical Expressive Power}\label{subsubsec_proof_kWL_kFWL}

Now we prove the results of the expressive power of higher-order transformers.

\begin{theorem}
    {\rm (\cref{Theorem_Akk<kWL} in the main text.)} Without additional input, $\mathcal A_k$ is strictly less expressive than $k$-WL and $k$-IGN.
\end{theorem}

\begin{proof}
    The proof is straightforward. Without any extra inputs (including positional encodings, structural encodings, and tuple indices), the initialization of every $k$-tuple is the same as in $k$-WL, see \citet{PPGN} for more details. Then according to Definition~\ref{def_k1k2-transformer} in the main text, tuples with identical initialization representations (colors) always obtain identical updates, suggesting that the color histogram always remains the same as the initial one. Consequently, the expressive power of $\mathcal A_k$ is the same as the $0$-th step (initialization) of $k$-WL. There exist non-isomorphic graphs that can be distinguished by stable $k$-WL histograms but not its initialization, yet the other direction does not hold. Therefore, $\mathcal A_k$ is strictly less expressive than $k$-WL. 
\end{proof}

As a corollary, since the initialization of ($k+1$)-WL is always strictly more powerful than the initialization of $k$-WL, we conclude that $\mathcal A_{k+1}$ is strictly more powerful than $\mathcal A_k$.

\begin{corollary}
    $\forall k\geq 1$, $\mathcal A_{k+1}$ is strictly more powerful than $\mathcal A_k$.
\end{corollary}

\begin{proof}
    We have already shown in the above proof of \cref{Theorem_Akk<kWL} that $\mathcal A_k$ is as powerful as the initialization as $k$-WL. It is known that for all $k\geq 1$, the initialization of $k+1$-WL is strictly more powerful than the initialization of $k$-WL. This holds because the initialization of ($k+1$)-WL can always distinguish whether there is a ($k+1$)-clique in the graph, while the initialization of $k$-WL fails.
\end{proof}

\begin{figure*}[t]
\vskip -0.1in
\begin{center}
\centerline{\includegraphics[width=0.5\columnwidth]{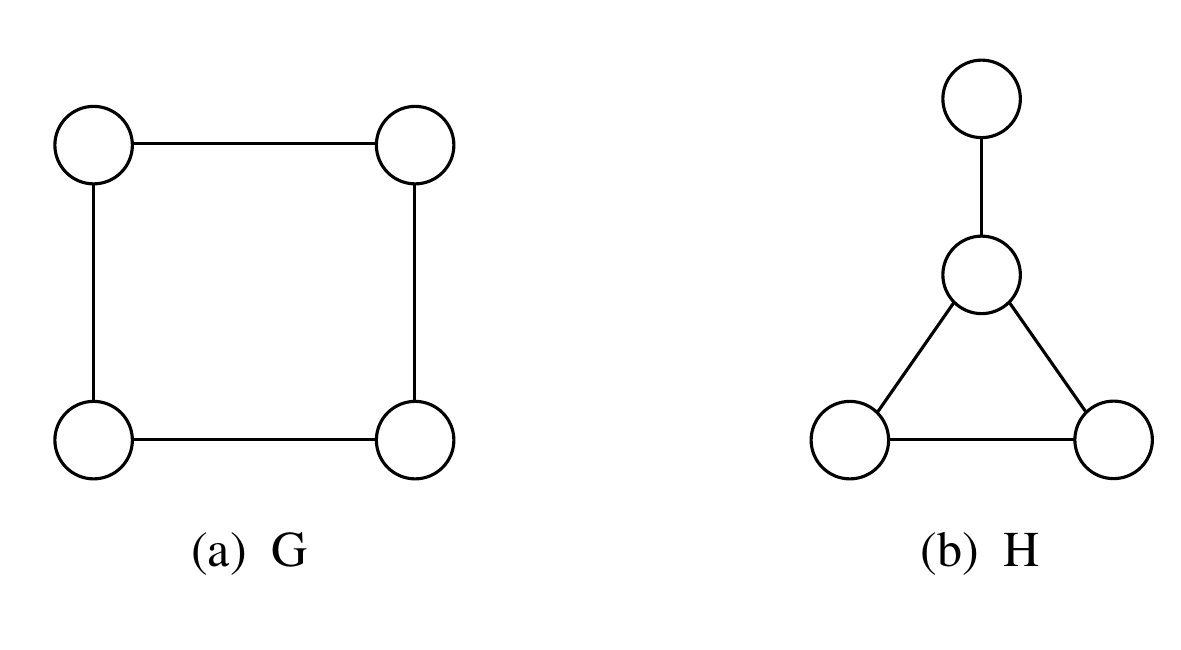}}
\vskip -0.1in
\caption{A pair of non-isomorphic graphs that can be distinguished by $1$-WL and $2$-WL, but cannot be distinguished by $\mathcal A_1$ and $\mathcal A_2$. For $k\geq 3$, both $k$-WL and $\mathcal A_k$ can distinguish them.}
\label{figure_example_nonisomorphic}
\end{center}
\vskip -0.1in
\end{figure*}

\cref{figure_example_nonisomorphic} provides an example of a pair of non-isomorphic graphs that can be distinguished by $1$-WL and $2$-WL, but cannot be distinguished by $\mathcal A_1$ and $\mathcal A_2$. Here we demonstrate the procedure. In the graph pairs, $n=4$. In our illustration, for simplicity, we use different alphabets $a, b,\dots$ to represent different features $\mX_\vi$ using a hash function. 

For $k=1$, we consider $1$-WL and $\mathcal A_1$ (i.e., the standard first-order graph transformer). In the initialization stage, both $1$-WL and $\mathcal A_1$ have identical $1$-tuple (node) feature initialization for both graphs $G$ and $H$: 

\begin{align}
    {\rm 1-WL}(G)^{(0)}=\mset{a,a,a,a}, \ {\rm 1-WL}(H)^{(0)}=\mset{a,a,a,a}\\
    \mathcal A_1(G)^{(0)}=\mset{a,a,a,a}, \ \mathcal A_1 (H)^{(0)}=\mset{a,a,a,a}
\end{align}

That is, both $1$-WL and $\mathcal A_1$ cannot distinguish $G$ and $H$ through their initialization. However, after one iteration, denote $b={\rm Hash}(a, \mset{a,a}), c={\rm Hash}(a, \mset{a}), d={\rm Hash}(a, \mset{a,a,a})$, and $e$ as the updated feature calculated by $\mathcal A_1$, we have

\begin{align}
    {\rm 1-WL}(G)^{(1)}=\mset{b,b,b,b}, \ {\rm 1-WL}(H)^{(1)}=\mset{b,b,c,d}\\
    \mathcal A_1(G)^{(1)}=\mset{e,e,e,e}, \ \mathcal A_1 (H)^{(1)}=\mset{e,e,e,e}
\end{align}

Therefore, $1$-WL can distinguish two graphs since all four nodes in $G$ have degree $2$, yet the degree histogram in $H$ is $\mset{2,2,1,3}$. In comparison, $\mathcal A_1$ has identical representation for $G$ and $H$. The histograms calculated by $\mathcal A_1$ are actually not updated compared with the initialization. By deduction, we can easily verify that for all iterations $t\geq 1$, we always have
\begin{equation}
    \mathcal A_1(G)^{(t)}=\mset{a^{(t)},a^{(t)},a^{(t)},a^{(t)}},\ 
    \mathcal A_1(H)^{(t)}=\mset{a^{(t)},a^{(t)},a^{(t)},a^{(t)}}
\end{equation}
which implies that $\mathcal A_1$ always fails in distinguishing $G$ and $H$. This observation is also in agreement with the fact that the plain first-order graph transformer cannot capture any edge information. 

For $k=2$ the same conclusion can be drawn. In the initialization stage, there are $n^2=16$ number of $2$-tuples. Denote $a=(i,i)$, $b=(i,j),i\neq j, e_{ij}=e_{ji}=1$ where $e_{ij}=e_{ji}=1$ indicates that there is an undirected edge between the nodes $i,j$, and finally $c=(i,j),i\neq j, e_{ij}=e_{ji}=0$ which implies that $i,j$ are not connected. Thus,

\begin{align}
    {\rm 2-WL}(G)^{(0)}=\mset{a,a,a,a,b,b,b,b,b,b,b,b,c,c,c,c},\ {\rm 2-WL}(H)^{(0)}=\mset{a,a,a,a,b,b,b,b,b,b,b,b,c,c,c,c}\\
    \mathcal A_2(G)^{(0)}=\mset{a,a,a,a,b,b,b,b,b,b,b,b,c,c,c,c},\  \mathcal A_2 (H)^{(0)}=\mset{a,a,a,a,b,b,b,b,b,b,b,b,c,c,c,c}
\end{align}

After one iteration, denote
\begin{align}
    d={\rm Hash}(a, (\mset{a,b,b,c}, \mset{a,b,b,c}))\\
    e={\rm Hash}(a, (\mset{a,b,c,c}, \mset{a,b,c,c}))\\
    f={\rm Hash}(a, (\mset{a,b,b,b}, \mset{a,b,b,b}))\\
    g={\rm Hash}(b, (\mset{a,b,b,c}, \mset{a,b,b,c}))\\
    h={\rm Hash}(b, (\mset{a,b,b,b}, \mset{a,b,c,c}))\\
    o={\rm Hash}(b, (\mset{a,b,c,c}, \mset{a,b,b,b}))\\
    p={\rm Hash}(b, (\mset{a,b,b,b}, \mset{a,b,b,c}))\\
    q={\rm Hash}(b, (\mset{a,b,b,c}, \mset{a,b,b,b}))\\
    r={\rm Hash}(c, (\mset{a,b,b,c}, \mset{a,b,b,c}))\\
    s={\rm Hash}(c, (\mset{a,b,c,c}, \mset{a,b,b,c}))\\
    u={\rm Hash}(c, (\mset{a,b,b,c}, \mset{a,b,c,c}))
\end{align}
Then $2$-WL gives
\begin{align}
    {\rm 2-WL}(G)^{(1)}=\mset{d,d,d,d,g,g,g,g,g,g,g,g,r,r,r,r}\\{\rm 2-WL}(H)^{(1)}=\mset{d,d,e,f,g,g,h,o,p,p,q,q,s,s,u,u}
\end{align}
The histogram of $G$ and $H$ are obvious different, therefore $2$-WL can distinguish them within one iteration, which is consistent with the fact that $2$-WL is as powerful as $1$-WL. 

In comparison, although $\mathcal A_k$ is able to incorporate edge information in the initialization stage, it is unable to learn further connectivity of the graph structure through the fully dense attention mechanism. Indeed, for $\forall t\geq 1$, the histogram of $\mathcal A_2$ after $t$ iterations always gives
\begin{align}
    \mathcal A_2(G)^{(t)}=\mset{a^{(t)},a^{(t)},a^{(t)},a^{(t)},b^{(t)},b^{(t)},b^{(t)},b^{(t)},b^{(t)},b^{(t)},b^{(t)},b^{(t)},c^{(t)},c^{(t)},c^{(t)},c^{(t)}}\\
    \mathcal A_2 (H)^{(t)}=\mset{a^{(t)},a^{(t)},a^{(t)},a^{(t)},b^{(t)},b^{(t)},b^{(t)},b^{(t)},b^{(t)},b^{(t)},b^{(t)},b^{(t)},c^{(t)},c^{(t)},c^{(t)},c^{(t)}}
\end{align}
hence always fail to distinguish two graphs.

For $k\geq3$, note that there is a $3$-clique in $H$ but not in $G$. Therefore, both $3$-WL and $\mathcal A_k$ can distinguish them via initialization.

In summary, a plain $\mathcal A_k$ is always strictly less expressive than $k$-WL and $k$-IGN in terms of distinguishing non-isomorphic graphs. However, for real-world graphs where the features are represented by a continuous-valued vector, the attention mechanism provides the potential for powerful representation learning. Particularly, with enhancements such as positional/structural encodings (PE/SE) and our sparse attention mechanisms, we can introduce asymmetry and variety into tuple features, obtaining powerful representations over graphs. See \cref{SectionProofAppendix} for more theoretical analysis, and \cref{SectionExperimentsAppendix} for more implementation details and practical benefits. 

Now we analyze the enhancement of augmenting inputs with tuple indices. As emphasized in our main text, this enhancement has potential drawbacks including breaking permutation invariance (which is also the case for \citep{PureTransformerspowerful}), hence is only of theoretical interest.

\begin{theorem}
    {\rm (\cref{theorem_Akk=kWL} in main text.)} For inputs $\mX \in \mathbb R^{n^k\times (d+k)}$ where each element is a concatenation of a $d$-dimensional tuple feature and $k$-dimension index of the tuple, one layer of $\mathcal A_k$ with hidden dimension $O(k)$ and $k$ heads augmented with input MLPs and residual connection can approximate one $k$-WL iteration arbitrarily well. If the softmax function is replaced by element-wise ReLU activation, $\mathcal A_k$ can exactly simulate $k$-WL.
\end{theorem}

\begin{proof}
    We breakdown the update function of $k$-WL \cref{equation_kWL} into several parts. The core idea of our construction is to select $k$-neighbors in \cref{equation_k-neighbor} via the multiplication of queries and keys, so that the update only involves the values (colors) of these neighboring tuples. As there are $k$ different neighbors for each tuple (depending on the position of indices being replaced), they are implemented by $k$ heads of transformer, and the concatenation of representations output by independent heads naturally implements the ordered set of neighbors as $k$ heads are allowed to be parameterized with different weights. Additionally, the color of the previous iteration and the hash function can be implemented by residual connection and feed forward layers (see \cite{PPGN} for more details). 

    First, we show that the $h$-th head with input MLPs $\phi_j:\mathbb R^{d+k}\rightarrow \mathbb R^{2(k-1)}$ and the softmax function ${\rm softmax}\big(\phi(\mX)Q (\phi(\mX) K)^\top\big)$ can approximate the selection of the $h$-th neighbor arbitrarily well. As we only want to reserve the index information in this part, inspired by \citep{RepresentationalStrengthsTransformer}, we define the tuple-wise MLP as

    \begin{align}
        \frac{1}{c}\phi_h(x_\vi)Q^h=\phi_h(x_\vi)K^h=\phi_h\big([z_\vi, \vi]\big)=\Big[&\cos(\frac{2\pi i_1}{M}),\sin(\frac{2\pi i_1}{M}),\dots,\cos(\frac{2\pi i_{h-1}}{M}),\sin(\frac{2\pi i_{h-1}}{M}),\\ &\cos(\frac{2\pi i_{h+1}}{M}),\sin(\frac{2\pi i_{h+1}}{M}), \dots, \cos(\frac{2\pi i_k}{M}),\sin(\frac{2\pi i_k}{M})\Big] \in \mathbb R^{2(k-1)}
    \end{align}

    where $\vi$ is the $k$-dimensional multi-index for the tuple, $z_\vi$ is the $d$-dimensional tuple feature being dropped for queries and keys, $M\geq n$ is a constant integer and $c\in\mathbb R$ is a positive constant. That is to say, we drop the $h$-th dimension of the index and reserve the remaining $k-1$ ones to find the tuples as candidates of $h$-th neighbor, who have identical indices as the query except the $h$-th index. Then for query tuple $\vi$ and potential target tuple $\vj$, we have

    \begin{equation}
        (\phi_h(x_\vi)Q^h)^\top \phi_h(x_\vj)K^h=c\sum_{l\in[k],l\neq h}\cos\Big(\frac{2\pi (i_l-j_l)}{M}\Big)
    \end{equation}

    Since $M\geq n$, the above value reaches its maximum $c(k-1)$ if and only if $i_l=j_l$ for all $l\in[k],l\neq h$, which exactly corresponds to the indices of $n$ tuples in the $h$-th neighbor of $\vi$. Then let $c\rightarrow\infty$, the output of the softmax would be

    \begin{equation}
        \bigg({\rm softmax} \Big((\phi_h(\mX)Q^h) (\phi_h(\mX)K^h)^\top\Big)\bigg)_{\vi, \vj}\rightarrow \left\{ 
        \begin{array}{ll} \frac{1}{n}, \ \ & \vj \in \mathcal N_h(\vi) \\ 
         0, \ \ & \vj \notin \mathcal N_h(\vi)
        \end{array}
        \right.
    \end{equation}

    Then the output of the $h$-th head attention would be
    \begin{equation}\label{equation_softmax_neighbor_update}
        \bigg({\rm softmax} \Big((\phi_h(\mX)Q^h) (\phi_h(\mX)K^h)^\top\Big) \mX V^h\bigg)_{\vi}\rightarrow \frac{1}{n}\sum_{\vj\in \mathcal N_h(\vi)} V^h
    \end{equation}

    where the values $x_\vj V^h$ only retain information of tuple features $z_\vj$ and drop the indices to represent the color of tuple $\vj$ in the last iteration (along with some possible injective transformation). Since $n$ is constant, the $\frac{1}{n}$ factor does not affect the injectivity. The above results show that we can successfully represent $\mathcal N_h(\vi)$, which is a multiset of $n$ different $k$-tuples. According to Remark 5, the summation operation can injectively represent the color of $\mathcal N_h(\vi)$ - although another MLP $\tau$ is needed before summing the color representations, this could be easily achieved by adding another input MLP (which is always the case whenever we need to use Remark 5). 

    It is remarkable that the above approximation can be exact if we replace the softmax with an element-wise ReLU activation: 

    \begin{equation}\label{equation_relu_neighbor_update}
        {\rm ReLU}\bigg((\phi_h(x_\vi)Q^h)^\top \phi_h(x_\vj)K^h -(c(k-1-\frac{1}{M^2}))\bigg)=\left\{ 
        \begin{array}{ll} \frac{c}{M^2}, \ \ & \vj \in \mathcal N_h(\vi) \\ 
         0, \ \ & \vj \notin \mathcal N_h(\vi)
        \end{array}
        \right.
    \end{equation}

    The rest of the proof is more straightforward. As we have $k$ different heads with different parameters, we repeat the above procedure to get all the $k$ neighbors of $\vi$. The set of neighbors can be represented by concatenating the output of these heads with the fixed order $h=1,\dots,k$. The color of $\vi$ in the previous iteration $\mC_\vi^t$ is reserved by residual connection. Then, according to Remark 3, the final hash function can be exactly implemented by concatenation of $\mC_\vi^t$ and \cref{equation_softmax_neighbor_update} (\cref{equation_relu_neighbor_update}). 
\end{proof}

Our construction and proof are different from \citep{PureTransformerspowerful}. Instead of directly simulating $k$-WL as we do, they aimed to approximate equivalence class basis in $k$-IGN. In detail, they proposed to augment tuple features with auxiliary inputs, namely 'node identifiers' and 'type identifiers' for $\mathcal A_k$ transformer, so that the latter can approximate equivalence class basis in $k$-IGN and thus $k$-WL. However, compared with their constructions, our version reveals at least the following advantages: (1) The input dimension of their construction is at least $O(d+kn)$ since they use orthogonal basis for each node as node identifiers, growing with $n$ makes it unrealistic. In comparison, our construction needs only $d+k$ input dimension and $O(k)$ hidden dimension, which are constants independent of $n$. Furthermore, since we need to index integers from $1$ to $n$, we need $k \log n$ bits in total; in comparison, since orthonormal vectors can consist of $0,1$, resulting $kn$ bits. 
(2) Their node identifiers serve as explicit features instead of only indexing, making their results stochastic and not permutation invariant. Actually, the role of their identifiers is the labeling method in \citep{RelationalklWL} to break the symmetry, but they do not perform relational pooling (summing over all possible permutations of labeling orders) to maintain permutation invariance of the models. In contrast, our output is deterministic, and if we use the exact parameterization of the ReLU version as in the proof above, the output of our transformer is permutation invariant, although generally the output is not permutation invariant for arbitrary parameterization. 
(3) The transformer parameters are required to be infinite in their construction, while we can avoid this problem with the ReLU activation version. We refer readers to \citep{PureTransformerspowerful} for more details of their methods.

\subsection{Proof and Additional Results for Section 4}\label{SubsecProofSec4}

\subsubsection{Kernelized Attention}

Literature so far has already applied linear/kernelized attention to standard graph transformers \citep{GPS}, yet rarely to high-order graph transformers. \citet{PureTransformerspowerful} indeed applied Performer to their TokenGT, but their TokenGT is actually not a strict second-order transformer since it apparently does not consider all $2$-tuples. Moreover, the theoretical guarantees of the kernelized transformers are still left unexplored, especially for high-order cases. We now formally state the definition of kernelized attention. Similar to the case for standard (one-dimensional) transformer, a single head $k$-order self-attention layer can be reformed as

\begin{equation}
    \mX_{i}^{(l+1)}=\sum_{j=1}^{n^k} \frac{\kappa(Q^{(l)}\mX_i^{(l)}, K^{(l)}\mX_j^{(l)})}{\sum_{k=1}^{n^k} \kappa(Q^{(l)}\mX_i^{(l)}, K^{(l)}\mX_k^{(l)})}\dot (V^{(l)}\mX_j^{(l)})
\end{equation}

where $\kappa:\mathbb R^d\times \mathbb R^d \rightarrow \mathbb R$ is the softmax kernel $\kappa(\mathbf x,\mathbf y)=\exp(\mathbf x^T\mathbf y)$. The kernel trick approximates the softmax via

\begin{equation}
    \kappa(\mathbf x,\mathbf y)=<\Phi(\mathbf x),\Phi(\mathbf y)>\approx \phi(\mathbf x)^T\phi(\mathbf y)
\end{equation}

where the first equation is by Mercer's theorem and the latter is a low-dimensional approximation with random transformation $\phi:\mathbb R^d\rightarrow \mathbb R^m$. In Performer \cite{Performer}, $\phi(\mathbf x)=\frac{\exp \big(-\frac{||X||_2^2}{2}\big)}{\sqrt m}[\exp(\mathbf w_1^T\mathbf x),\dots, \exp(\mathbf w_m^T\mathbf x)]$ where $\mathbf w_k\sim \mathcal N(0,I_d)$. In Linear Transformer \cite{LinearTransformerRNN}, $\phi(\mathbf x)=elu(\mathbf x)+1$. Then we can further rewrite the attention as

\begin{equation}\label{kernelized_Ak}
    \mX_{i}= \frac{\Big(\phi(\mX_i Q)^\top \sum_{j=1}^{n^k}\big( \phi(\mX_j K) \otimes (\mX_j V)\big) \Big)^\top}{\phi(\mX_i Q)^\top \sum_{l=1}^{n^k} \phi(\mX_l K)}
\end{equation}

where $\otimes$ is the outer product. Note that the two summations $\sum_{j=1}^{n^k}\big( \phi(\mX_j K) \otimes (\mX_j V)\big)\in\mathbb R^{md'}$ and $\sum_{l=1}^{n^k} \phi(\mX_l K)$ are shared for all query tokens, while the former has the bottleneck complexity $O(n^kmd')$. Again, since $m,d'$ are of the same magnitude as $d$, the total time complexity of kernelized $\mathcal A_k$ is $O(n^kd^2)$.

Now we prove the approximation results of $k$-IGN to kernelized $\mathcal A_k$ in the main text.

\begin{theorem}
    {\rm (\cref{Theorem_kIGN>linearA_kk11} in main text.)} $k$-IGN can approximate kernelized $\mathcal A_k$ with Linear Transformer or Performer architectures arbitrarily well. $\mathcal A_k$ with Linear Transformer or Performer architectures is strictly less powerful than $k$-IGN.
\end{theorem}

\begin{proof}
    First, we emphasize the fact that although having $n^k$ input tokens, a pure $\mathcal A_k$ treats all tuples homogeneously regardless of equivalence classes. Consequently, the $\mathcal A_k$ with Linear Transformer or Performer architectures is of the same form as DeepSets and Sumformer\citep{Sumformer}.

    To show that $k$-IGN can approximate kernelized $\mathcal A_k$, we construct identical weights for all equivalence classes in $k$-IGN, namely $w_\mu=w, b_\lambda=b, \forall \mu\in[{\rm bell}(2k)], \forall \lambda \in [{\rm bell}(k)]$. Then the update function in \cref{equation_kIGN} (in main text) reduces to

    \begin{equation}\label{kIGN-homo}
        L_{k\rightarrow k}(\mX)_\vi = \sum_\vj \mX_\vj w + b
    \end{equation}

    Compared with MPNN of the first order, $k$-IGN aggregates global information from all $n^k$ tuples by design, hence can recover the role of ``virtual node'' in MPNN~\citep{ConnectionMPNNGT}. We will show that \cref{kIGN-homo} with residual connection and output MLPs can approximate kernelized $\mathcal A_k$ with Linear Transformer and Performer architectures arbitrarily well in $O(1)$ depth and $O(1)$ width. A similar technique is used in our proof of \cref{Proposition_virtual_tuple=Linear_A_kk11} and \citep{ConnectionMPNNGT}.

    We now consider another $k$-IGN layer along with an input MLP. The input MLP computes $\phi(\mX_\vj K)\otimes (\mX_\vj V)$ for all $\vj$, and the $k$-IGN layer computes $\sum_\vj \phi(\mX_\vj K)\otimes (\mX_\vj V)$.

    Under the compactness assumption of inputs and weight matrices (stated in \cref{SubsubsecNotation}), we consider the following construction of two $k$-IGN layers along with residual connection and output MLPs $\psi$. Mathematically,

    \begin{equation}
        \mX^{(new)}_\vi := \psi\big(\mX_\vi, \sum_\vj \mX_\vj w + b, \sum_\vj \phi(\mX_\vj K)\otimes (\mX_\vj V)\big)=\frac{\Big(\phi(\mX_\vi Q)^\top \sum_\vj\big( \phi(\mX_\vj K) \otimes (\mX_\vj V)\big) \Big)^\top}{\phi(\mX_\vi Q)^\top \sum_\vl \phi(\mX_\vl K)}
    \end{equation}
    
    where the first equality is our construction and the second equality is the target that we aim to approximate arbitrarily well with $O(1)$ width and $O(1)$ depth. Analogously to \citep{ConnectionMPNNGT}, by the uniform continuity of the functions, it suffices to show that 1) we can approximate $\phi$, 2) we can approximate multiplication and vector-scalar division, 3) the denominator $\phi(\mX_\vi Q)^\top \sum_\vl \phi(\mX_\vl K)$ is uniformly lower bounded by a positive number for any node features.

    For 1), each component of $\phi$ (in both Performer and Linear Transformer)is continuous, and all inputs $\mX_\vj Q, \mX_\vj K$ lie in compact domain. Therefore, $\phi$ can be approximated arbitrarily well by MLP with $O(1)$ width and $O(1)$ depth~\citep{Approximationsuperposition}.

    For 2), since multiplication and vector-scalar division operations are all continuous, it suffices to show that all operands lie in a compact domain. This is true since $\mX^{(0)}$ and $Q,K,V$ are all compact, $\phi$ is continuous and that $n$ is fixed. Lastly, since all these operations do not involve $n$, the depth and width of MLPs are constant in $n$.

    For 3), we need to show that the denominator is bound by a positive constant. In Performer, $\phi(\mX)=\frac{\exp \big(-\frac{||X||_2^2}{2}\big)}{\sqrt m}[\exp(\mathbf w_1^T\mX),\dots, \exp(\mathbf w_m^T\mX)]$ where $\mathbf w_k\sim \mathcal N(0,I_d)$. As $||\mathbf w_i^T\mX||\leq ||\mathbf w_i||\cdot ||\mX||$, which implies that $\exp (\mathbf w_i^T \mX)$ is lower bounded by $\exp (-||\mathbf w_i||\cdot ||\mX||)$. Consequently, the demonimator $\phi(\mX_\vi Q)^\top \sum_\vl \phi(\mX_\vl K)$ is lower bounded. For Linear Transformer, the proof is essentially the same as Performer. It boils down to showing that $\phi(\mathbf x)=elu(\mathbf x)+1$ is continuous and positive, which is indeed the case.

    In conclusion, $k$-IGN along with residual connection and output MLPs can approximate kernelized $\mathcal A_k$ with Linear Transformer or Performer architectures arbitrarily well.


    For the second part that kernelized $\mathcal A_k$ is strictly less expressive than $k$-IGN, recall that we have already shown that even $\mathcal A_k$ without kernelization is strictly less expressive than $k$-WL (and thus $k$-IGN). Kernelized $\mathcal A_k$ (either with Linear Transformer or Performer architectures) cannot be more expressive than $\mathcal A_k$ without kernelization as we can always find a set of parameters of $\mathcal A_k$ to implement a given kernelized $\mathcal A_k$. Consequently, kernelized $\mathcal A_k$ with Linear Transformer or Performer architectures is strictly less powerful than $k$-IGN. 
\end{proof}

Note that Linear equivalence layers in \cref{kIGN-homo} augmented with input and output MLPs suffices a Sumformer~\citep{Sumformer} by definition,
which can simulate $\mathcal A_k$ (without kernelization) within $O(1)$ depth and $\big({}^{n^k+d}_{\ \ d} \big)-1$ width for continuous function~\citep{Sumformer}. The same conclusion can be drawn using the concept of Deepset~\citep{DeepSet, UniversalEquiSet}. Although the computation complexity is not practical (note that $k$-IGN and $k$-PPGN require the same complexity to achieve full expressivity as well~\citep{PPGN}), the theoretical implication is interesting: it verifies our previous conclusion that $\mathcal A_k$ cannot be more expressive than $k$-IGN.

\subsubsection{Sparse Attention Mechanisms} \label{SubsubsecProofSparse}

As we emphasize in the main texts, the plain global attention $\mathcal A_k$ has limited expressive power due to the homogeneous dense attention scheme (\cref{Theorem_Akk<kWL}). We now provide proofs for the expressive power of our proposed sparse attention mechanisms.

\paragraph{Neighbor Attention.} Analogous to $k$-WL, neighbor attention mechanism benefits from locality and sparsity, revealing expressive power with a lower bound of $k$-WL. We first restate the definition of neighbor atention,
\begin{equation}
    \Big(\mathcal A_k^{\mathsf {Ngbh}}(\mX, \mX)\Big)_{\vi}={\rm Concat}\bigg[{\rm softmax} \Big((x_\vi Q^j) \big(x_{[\psi_j(\vi, u)|u\in[n]]} K^j\big)^\top \Big) \big(x_{[\psi_j(\vi, u)|u\in[n]]} V^j\big) \Big| j\in[k] \bigg]
\end{equation}
where $\psi_j(\vi, u)$ means replacing the $j$-th element in $\vi$ with $u$, and thus $x_{[\psi_j(\vi, u)|u\in[n]]}\in \mathbb R^{n\times d},\ j\in[k]$ is just the feature of the $j$-th neighbor of tuple $\vi$.

\begin{theorem}
    {\rm (\cref{TheoremNeighborAttn>=kWL} in main text.)} Neighbor attention $\mathcal A_k^{\mathsf {Ngbh}}$ with residual connection, output MLPs and $k$ heads is as powerful as $k$-WL.
\end{theorem}

\begin{proof}
    We first show that $\mathcal A_k^{\mathsf {Ngbh}}$ simulates $k$-WL by constructing a set of parameter weights. Define
    \begin{align}
        x_\vi Q^j=\vec 1,\vi\in \mathbb R^{n^k},j\in[k]\\
        x_\vi K^j=\vec 1,\vi\in \mathbb R^{n^k},j\in[k]
    \end{align}
    where $Q^j, K^j, V^j (j\in[k])$ are weight parameters of the $j$-th head. Note again that we always allow bias terms for $Q,K,V$ matrices in practical transformers; thus, the above equation can be easily obtained by setting weight matrices as zero matrices and setting the bias as $\vec 1\in \mathbb R^{d_k}$, where $d_k$ is the hidden dimension of query and keys which is a constant.
    
    Then the neighbor attention reduces to
    \begin{equation}
        \Big(\mathcal A_k^{\mathsf {Ngbh}}(\mX, \mX)\Big)_{\vi}={\rm Concat}\bigg[\frac{1}{n} \sum_{u\in[n]} x_{\psi_j(\vi, u)} V^j \Big| j\in[k] \bigg]
    \end{equation}
    In graph-level isomorphism task, we consider a fixed graph pair, thus $n$ is a fixed constant for our interested graphs, which will not affect the ability to implement hash function. Therefore, the neighbor attention aggregates information from tuples in $j$-th $k$-neighbor through the $j$-th head, for $j\in[k]$ respectively, which is exactly what $k$-WL does in \cref{equation_kWL}. Then according to \citep{PPGN} (and similar to the techniques we used in \cref{SectionTheory}), the concat operation, residual connection and output MLPs implement an ordered set $()$, $c_k^{t-1}(\vv, G)$ and the hash function in \cref{equation_kWL_color} (in the main text) or \cref{equation_kWL}, thus completely simulating $k$-WL with proper constructions. 

    The other side is because the attention described in \cref{equation_neighbor_attention} (in the main text) is always an instance of \cref{equation_kWL_color} in the main text as long as we restrict the reception field to the $k$-neighbors.

\end{proof}

In other words, $\mathcal A_k^{\mathsf {Ngbh}}$ is actually a form of $k$-WL implementation. Although we need more layers of $\mathcal A_k^{\mathsf {Ngbh}}$ (and more $k$-WL iterations) to simulate one layer $k$-IGN, the neighbor attention is much lower in complexity while having the same expressive power. Particularly, since one $\mathcal A_k^{\mathsf {Ngbh}}$ layer is equivalent to one $k$-WL iteration, and that $(k-1)$ iterations of $k$-WL simulates one $k$-IGN layer~\citep{ExpressivepowerkIGN}. Therefore, we need $(k-1)$ layers of $\mathcal A_k^{\mathsf {Ngbh}}$ to simulate one $k$-IGN layer, leading to $O(n^{k+1}k(k-1)d)$ complexity (recall that in $\mathcal A_k^{\mathsf {Ngbh}}$ each $k$-tuple computes attention with $k$ $k$-neighbors); it is still more efficient than one $k$-IGN with $O(n^{2k}d)$ complexity, though. Additionally, as we discuss in the main text, although attention does not increase expressive power compared to $k$-WL, it may improve real-world task performance due to more flexible attention-based aggregation compared to naive homogeneous aggregation. A natural extension would be a dense $\mathcal A_k$ with complete tuple adjacent information or equivalence class basis is as powerful as $k$-WL and thus $k$-IGN, yet the attention may lead to stronger real-world performance. 

\paragraph{Local Neighbor Attention.} Now we move to the local neighbor attention $\mathcal A_k^{\mathsf {LN}}$. Before we state our attention mechanism, we first give a brief summary of the algorithms in \citep{WLgoSparse}.

\citet{WLgoSparse} proposed a family of $\delta\text{-}k$-dimensional WL algorithms, which is more powerful than $k$-WL. Additional to the $k$-neighbor in $k$-WL, $\delta\text{-}k$-WL augments each $j$-th $k$-neighbor with the connectivity between the node being replaced $\vi_j$ and the $n$ nodes replacing it. Formally,
\begin{equation}\label{equation_deltakWL}
    \delta\text{-}k\text{-}{\rm WL:} \mC_{\vi}^{t+1}={\rm hash}\Bigg(\mC_{\vi}^t, \bigg(\mset{\mC_{\vj}^t, {\rm adj}(\vi_j, \vu_j)\big| \vu \in \mathcal N_j(\vi)}\Big|j\in [k] \bigg)\Bigg)
\end{equation}

Based on $\delta\text{-}k$-WL, we can easily extend $\mathcal A_k^{\mathsf {Ngbh}}$ to $\mathcal A_k^{\mathsf {Ngbh+}}$ by incorporating additional structural information ${\rm adj}(\vi_j, \vu_j)$ via attention reweighting or attention bias, see \cref{SubsecImplementationAppendix} for more details. Our $\mathcal A_k^{\mathsf {Ngbh+}}$ can be regarded as the attention version of $\delta\text{-}k$-WL and they are equivalent in expressive power, while both are strictly more expressive than $k$-WL and $\mathcal A_k^{\mathsf {Ngbh}}$. Our experiments on both synthetic datasets and real-world datasets verify the advantage of $\mathcal A_k^{\mathsf {Ngbh+}}$ over $\mathcal A_k^{\mathsf {Ngbh}}$: the former universally reveal better performance without increasing much computation cost: both $\mathcal A_k^{\mathsf {Ngbh}}$ and $\mathcal A_k^{\mathsf {Ngbh+}}$ have $O(n^{k+1})$ complexity regarding $n$.

Now we continue to explain our local neighbor attention mechanism. Based on $\delta\text{-}k$-WL, \citet{WLgoSparse} further proposed a local variant of WL, namely $\delta\text{-}k$-LWL, which defines the $j$-th \textit{local neighbor} of tuple $\vi$:

\begin{equation}
    \mathcal N_j^{Local}(\vi)=\mset{\psi_j(\vi, v)|v\in N(\vi_j)}
\end{equation}

where $\psi_j(\vi, v)$ means replacing the $j$-th element in $\vi$ with $v$, and $N(\vi_j)$ refers to the (one-dimensional) neighbors of node $\vi_j$. $\delta\text{-}k$-LWL is defined as
\begin{equation}
    \label{equation_deltakLWL}
    \delta\text{-}k\text{-}{\rm LWL:} \mC_{\vi}^{t+1}={\rm hash}\Bigg(\mC_{\vi}^t, \bigg(\mset{\mC_{\vj}^t \big| \vj \in \mathcal N_j^{Local}(\vi)}\Big|j\in [k] \bigg)\Bigg)
\end{equation}

\citet{WLgoSparse} showed that $\delta\text{-}k$\text{-}WL is at least as powerful than $\delta\text{-}k$\text{-}LWL (the gap would be a histogram of colors of full $k$-neighbor, as $\delta\text{-}k$\text{-}WL is exactly as powerful as $\delta\text{-}k$\text{-}LWL+, while the latter incorporates histogram of $k$-neighbor colors compared with $\delta\text{-}k$\text{-}LWL. Interested readers please refer to \citet{WLgoSparse} for more details). While $\delta\text{-}k$\text{-}WL is strictly more powerful than $k$-WL, the relation between $\delta\text{-}k$\text{-}LWL and $k$-WL is unknown due to the aforementioned gap. However, a significant advantage of $\delta\text{-}k$\text{-}LWL over $k$-WL is its lower complexity due to the sparsity of local neighbors compared with (full or global) $k$-neighbors, reducing the complexity from $O(n^{k+1})k$ to $O(n^{k}\bar D k)$, where $\bar D$ is the average node degree in the graph.

And now we are going to show that our local neighbor attention is at least as powerful as $\delta\text{-}k$-LWL. Recall that the ($k$-head) local neighbor attention $\mathcal A_k^{\mathsf {LN}}$ is defined as
\begin{equation}
    \Big(\mathcal A_k^{\mathsf {LN}}(\mX, \mX)\Big)_{\vi}={\rm Concat}\bigg[{\rm softmax} \Big((x_\vi Q^j) \big(x_{[\psi_j(\vi, u)|u\in N(\vi_j)]} K^j\big)^\top \Big) \big(x_{[\psi_j(\vi, u)|u\in N(\vi_j)]} V^j\big) \Big| j\in[k] \bigg]
\end{equation}

where $N(\vi_j)$ denotes the set of neighbors of the $j$-th element in tuple $\vi$, and $\psi_j(\vi, u)$ still refers to replacing the $j$-th element in $\vi$ with $u$.

\begin{theorem}
    {\rm (\cref{Theorem_LocalNA>=deltakLWL} in main text.)} Local neighbor attention $\mathcal A_k^{\mathsf {LN}}$ with residual connection, output MLPs and $k$ heads is at least as powerful as $\delta\text{-}k$-LWL.
\end{theorem}

\begin{proof}
    We use the same construction as in our proof for $\mathcal A_k^{\mathsf {Ngbh}}$. Define
    \begin{align}
        x_\vi Q^j=\vec 1\in \mathbb R^{d_k},\vi\in \mathbb R^{n^k},j\in[k]\\
        x_\vi K^j=\vec 1\in \mathbb R^{d_k},\vi\in \mathbb R^{n^k},j\in[k]\\
    \end{align}
    where $d_k$ is still the constant latent dimension. Then the local neighbor attention reduces to
    \begin{equation}
        \Big(\mathcal A_k^{\mathsf {LN}}(\mX, \mX)\Big)_{\vi}={\rm Concat}\bigg[\frac{1}{d(\vi_j)} \sum_{u\in N(\vi_j)} x_{\psi_j(\vi, u)} V^j \Big| j\in[k] \bigg]
    \end{equation}
    where $d(\vi_j)$ is the degree of node $\vi_j$. This is an instance of $\delta\text{-}k$-LGNN~\citep{WLgoSparse} with mean aggregation, which has the same expressive power as $\delta\text{-}k$-LWL~\citep{WLgoSparse}. 
    
    We can also provide proof for local neighbor attention with slight modifications under our theoretical framework as in \citep{PPGN}. According to Remark 5, if we want to keep the representation of multisets injective, a summation instead of the ``averaging'' operation is needed, i.e. we want to eliminate the $d(\vi_j)$ factor:
    \begin{equation}\label{equation_reduced_ln}
        \Big(\mathcal A_k^{\mathsf {LN}}(\mX, \mX)\Big)_{\vi}={\rm Concat}\bigg[\sum_{u\in N(\vi_j)} x_{\psi_j(\vi, u)} V^j \Big| j\in[k] \bigg]
    \end{equation}
    Several solutions are applicable to obatin \cref{equation_reduced_ln}: (1) replace the softmax in local neighbor attention with element-wise relu activation, so that the local neighbor attention changes to summation over tuples in each local neighbor; (2) attach node degrees as part of the input, so that the output MLPs can multiply $d(\vi_j)$ back to recover the summation operation. 
    Hence, the local neighbor attention exactly aggregate information from tuples in $j$-th local neighbor in the $j$-th head, for $j\in[k]$ respectively. Then according to Remark 3 and Remark 5, the concat operation, residual connection, and output MLPs completely simulate $\delta\text{-}k$-LWL in \cref{equation_deltakLWL} with proper constructions. 

\end{proof}

A direct corollary is that $\mathcal A_k^{\mathsf {LN}}$ is strictly less expressive than $\mathcal A_k^{\mathsf {Ngbh+}}$, since $\delta\text{-}k$-LWL is strictly less expressive than $\delta\text{-}k$-WL (while the latter is strictly more expressive than $k$-WL). However, the relation between $\mathcal A_k^{\mathsf {LN}}$ and $k$-WL is unknown: there are some non-isomorphic graph pairs that $\mathcal A_k^{\mathsf {LN}}$ can distinguish but $k$-WL fails (see the CSL experiment in our main text), while the other direction is still an open question.

\paragraph{Virtual Tuple Attention.} Finally we provide proofs for the virtual tuple attention mechanism. Recall the definition of virtual tuple attention $\mathcal A_k^{\mathsf {VT}}$ is
\begin{align}
    \mathcal A_k^{\mathsf {VT}}(\mX', \mX')_{n^k+1}&= {\rm softmax}\Big((x'Q^1) (\mX K^1)^\top\Big) \mX V^1   \\
    \mathcal A_k^{\mathsf {VT}}(\mX', \mX')_{\vi}&= x' V^{2}
\end{align}
where the feature of virtual tuple is denoted as $x'$, and the input is augmented to $\mX'=[\mX, x']\in\mathbb R^{(n^k+1)\times d}$.

\begin{proposition}
    {\rm (Proposition~\ref{Proposition_virtual_tuple=Linear_A_kk11} in main text.)} $O(1)$ depth and $O(1)$ width virtual tuple attention $\mathcal A_k^{\mathsf {VT}}$ can approximate $\mathcal A_k$ with Performer or Linear-Transformer architecture arbitrarily well. 
\end{proposition}

\begin{proof}
    As stated before, all tuples are treated as one single equivalence class in $\mathcal A_k$, its calculation in the Performer or Linear-Transformer architecture is the same as the standard transformer $\mathcal A_1$ except for the number of tuples. Therefore, $\mathcal A_k$ with kernel tricks can be directly reformed to the normal Performer or Linear-Transformer with $n^k$ inputs from the same equivalence class. Hence the $O(1)$ depth and $O(1)$ width simplified MPNN + virtual tuple~\citep{ConnectionMPNNGT} operating on $n^k$ input tokens can approximate $\mathcal A_k$ with Performer or Linear-Transformer architecture arbitrarily well. As in \citep{ConnectionMPNNGT}, here 'simplified' suggests that we ignore message passing within real node neighbors, and the update function is parameterized with heterogeneous parameters for virtual tuples and real tuples. However, all real tuples are treated as the same equivalence class.
    
    Therefore, we only need to show that virtual tuple attention can recover simplified message passing neural networks (operating on homogeneous $k$-tuples) + virtual tuple. To this end, we simply define
    \begin{equation}
        x'Q^1=1, \mX K^1=\vec 1_{n^k}
    \end{equation}

    The virtual tuple attention reduces to 
    \begin{align}
        \mathcal A_k^{\mathsf {VT}}(\mX', \mX')_{n^k+1}&= \frac{1}{n^k} \sum_{i=1}^{n^k} \mX_i V^1   \\
        \mathcal A_k^{\mathsf {VT}}(\mX', \mX')_{\vi}&= x' V^{2}
    \end{align}

    When the above update function is augmented with input and output MLPs, it is exactly reduced to the update function of the simplified MPNN + virtual tuple with mean aggregation function~\citep{ConnectionMPNNGT} which treats all tuples as one equivalence class. 

    Denote input features as $\mX^{(0)}$, the kernelized $\mathcal A_k$ computes (the same as in \cref{kernelized_Ak})
    \begin{equation}
        \mX_{i}^{(new)}= \frac{\Big(\phi(\mX_i^{(0)} Q)^\top \sum_{j=1}^{n^k}\big( \phi(\mX_j^{(0)} K) \otimes (\mX_j^{(0)} V)\big) \Big)^\top}{\phi(\mX_i^{(0)} Q)^\top \sum_{l=1}^{n^k} \phi(\mX_l^{(0)} K)}
    \end{equation}

    Under the compactness assumption of the inputs and weight matrices (stated in \cref{SubsubsecNotation}), we consider the following construction of two layers of virtual tuple attention (along with the residual connection and MLP). Intuitively, in the first layer (denoted by $(1)$), we first process each real node with an input MLP $\theta$ to compute $\theta(\mX_j^{(0)}):={\rm ReshapeTo1D}\big(\phi(\mX_j^{(0)} K) \otimes (\mX_j^{(0)} V)\big)$; then through the first attention computation, the virtual tuple has the feature $x'^{(1 )}=\frac{1}{n^k} \sum_{i=1}^{n^k} [\mX_i^{(0)}, \theta(\mX_i^{(0)})] V^{1(1)}$, where $[\cdot, \cdot]$ means concatenation, and the the input dimension for $V^{1(1)}$ would be $(m+1)d$ instead of $d$. We will show that along with the output MLPs $\psi^{(1)}$ and the feature $x'^{(1)}$, the virtual tuple can approximate $\sum_{j=1}^{n^k}\big( \phi(\mX_j^{(0)} K) \otimes (\mX_j^{(0)} V)\big)$ and $\sum_{l=1}^{n^k} \phi(\mX_l^{(0)} K)$. Then in the second layer denoted by $(2)$, the virtual tuple sends the message back to each real tuple, and each real tuple $\vi$ can approximate \cref{kernelized_Ak} along with the feature $x'^{(1)}V^{2(2)}$, $\mX_\vi^{(0)}$ (via residual connection) and the output MLPs $\psi^{(2)}$.

    In detail, in the first layer along with the MLPs $\psi^{(1)}$, the virtual tuple updates as follows,
    \begin{equation}
        x'^{(1)}:=\psi^{(1)}(\frac{1}{n^k} \sum_{i=1}^{n^k} [\mX_i^{(0)}, \theta(\mX_i^{(0)})] V^{1(1)})=\Big[\sum_{l=1}^{n^k} \phi(\mX_l^{(0)} K), {\rm ReshapeTo1D} \big(\sum_{j=1}^{n^k} \phi(\mX_j^{(0)} K) \otimes (\mX_j^{(0)} V)\big)\Big]
    \end{equation}
    where $\phi(\mX_j^{(0)} K) \otimes (\mX_j^{(0)} V)\in \mathbb R^{md'}$ is reshaped into $1$-dimensional feature vector by ${\rm ReshapeTo1D}$ in raster order. Therefore, the final dimension of the virtual tuple feature $x'^{(1)}$ is $m(d'+1)$. 

    In the second layer, each real tuple $\vi$ receives information $x'^{(1)}$ from the virtual tuple. The residual connection reserve the input feature $\mX_\vi^{(0)}$, then with MLPs $\psi^{(2)}$
    \begin{equation}
        \mX_\vi^{(2)}:=\psi^{(2)}(\mX_\vi^{(0)}, x'^{(1)}V^{2(2)})=\frac{\Big(\phi(\mX_i^{(0)} Q)^\top \sum_{j=1}^{n^k}\big( \phi(\mX_j^{(0)} K) \otimes (\mX_j^{(0)} V)\big) \Big)^\top}{\phi(\mX_i^{(0)} Q)^\top \sum_{l=1}^{n^k} \phi(\mX_l^{(0)} K)}
    \end{equation}
    
    We need to show that the above equations can be approximated arbitrarily well by MLPs $\psi^{(1)}, \psi^{(2)}$ with $O(1)$ width and $O(1)$ depth. By the uniform continuity of the functions, it suffices to show that 1) we can approxiate $\phi$, 2) we can approximate multiplication and vector-scalar division, 3) the denominator $\phi(\mX_i^{(0)} Q)^\top \sum_{l=1}^{n^k} \phi(\mX_l^{(0)} K)$ is uniformly lower bounded by a positive number for any node features.

    For 1), each component of $\phi$ (in both Performer and Linear Transformer)is continuous, and all inputs $\mX_j^{(0)} Q, \mX_j^{(0)} K$ lie in compact domain. Therefore, $\phi$ can be approximated arbitrarily well by MLP with $O(1)$ width and $O(1)$ depth~\citep{Approximationsuperposition}.

    For 2), since multiplication and vector-scalar division operations are all continuous, it suffices to show that all operands lie in a compact domain. This is true since $\mX^{(0)}$ and $Q,K,V$ are all compact, $\phi$ is continuous, and $n$ is fixed. Lastly, since all these operations do not involve $n$, the depth and width of the MLPs are constant in $n$.

    For 3), we need to show that the denominator is bound by a positive constant. In Performer, $\phi(\mX)=\frac{\exp \big(-\frac{||X||_2^2}{2}\big)}{\sqrt m}[\exp(\mathbf w_1^T\mX),\dots, \exp(\mathbf w_m^T\mX)]$ where $\mathbf w_k\sim \mathcal N(0,I_d)$. As $||\mathbf w_i^T\mX||\leq ||\mathbf w_i||\cdot ||\mX||$, which implies that $\exp (\mathbf w_i^T \mX)$ is lower bounded by $\exp (-||\mathbf w_i||\cdot ||\mX||)$. Consequently, the demonimator $\phi(\mX_i^{(0)} Q)^\top \sum_{l=1}^{n^k} \phi(\mX_l^{(0)} K)$ is lower bounded. 
    
    For Linear Transformer, the proof is essentially the same as Performer. It boils down to showing that $\phi(\mathbf x)=elu(\mathbf x)+1$ is continuous and positive, which is indeed the case.
    
\end{proof}

\section{SIMPLICIAL TRANSFORMERS}\label{SectionSimplicialTransformer}

In this section, we will detail the theoretical properties and our designs for simplicial transformers. This part provides results additional to the main text (where we do not discuss simplicial transformers in detail due to limited space), yet can be regarded as a natural extension. Simplicial complex is a different concept compared to the graph, which we will introduce in detail below. However, as discussed in the main text, $k$-simplices are generally more sparse than $k+1$-tuples and are always a subset of the latter. Hence, simplices can be regarded as a result of sampling tuples, and the simplicial (complexes) transformer is a sparse variant of the tuple-based transformer we discussed in the main text. Moreover, the sparse attention mechanisms and the design principles we proposed for tuple-based transformers can be naturally extended to simplicial transformers with slight modifications, resulting in more efficient models that are still powerful.

\subsection{Background of Algebraic Topology Theories}\label{SubsecHodgeTheory}

To provide readers with basic knowledge related to simplicial complexes, we introduce some existing fundamental algebraic topology theories in this subsection.


An abstract simplicial complex, denoted as $\mathcal{K}$, is defined over a finite set $V$, which comprises subsets of $V$ adhering to the property of closure under inclusion. Specifically, $V$ represents a set of vertices, denoted by $[n] = {1, 2, \ldots, n}$. A subset within $\mathcal{K}$, having cardinality $k+1$, is termed a $k$-simplex. To illustrate, vertices are $0$-simplices, directed edges are $1$-simplices, and oriented triangles or 3-cliques are $2$-simplices. The set of all $k$-simplices in $\mathcal{K}$ is represented as $S_k(\mathcal{K})$. A $k$-simplices has a dimension of $k$, and the dimension of the complex $\mathcal{K}$ itself is the maximum dimension among all its faces.

When two $(k+1)$-simplices share a common $k$-face (a subset of a simplex), they are termed as $k$-down neighbors. Conversely, two $k$-simplices that share a $(k+1)$-simplex are known as $(k+1)$-up neighbors. Furthermore, a $k$-cochain or $k$-form is a function defined on $\mathcal K_{k+1}$, $f:V\times \dots \times V\rightarrow \mathbb R$ that is equivariant to the permutation. Although $k$-cochains have the structure of vector spaces, they are usually called cochain groups $\mathcal C^k(\mathcal K,\mathbb R)$. Chain groups $\mathcal C_k(\mathcal K,\mathbb R)$ are defined as duals of co-chain groups. 

The simplicial coboundary maps $\delta_k: \mathcal C^{k}(\mathcal K,\mathbb R)\rightarrow \mathcal C^{k+1}(\mathcal K,\mathbb R) $ is defined as
\begin{equation}
    (\delta_k f)([v_0,\dots,v_{i+1}])=\sum_{j=0}^{k+1}(-1)^j f([v_0, \dots, \hat V^j,\dots,v_{k+1}])
\end{equation}
where $\hat V^j$ suggests that the vertex $v_j$ is omitted. Further, we can define the adjoint of coboundary operator: $\delta_k^*: \mathcal C^{k+1}(\mathcal K,\mathbb R)\rightarrow \mathcal C^{k}(\mathcal K,\mathbb R)$.

Utilizing the concept of boundary and coboundary, the Hodge $k$-Laplacian operator (also called the combinatorial Laplace operator) is defined as:
\begin{equation}
    \mathbf L_k=\mathbf L_{k,down}+\mathbf L_{k,up}=\delta_{k-1}\delta_{k-1}^*+\delta_k^*\delta_k
\end{equation}
where we omit the reference to simplicial complex $\mathcal K$ from the notation for simplicity. By definition, all three operators $\mathbf L_k,\mathbf L_{k,up},\mathbf L_{k,down}$ are self-adjoint, nonnegative and compact. 

In the Hilbert space, the matrix representation for boundary and co-boundary operators are equivalent to adjacent matrix of $k$ and $k+1$ order simplices. We write the matrix representation for $\delta_k^*$ as $\mathbf B_{k+1}\in\mathbb R^{|S_{k}|\times |S_{k+1}|}$ (one can view it as the adjacent matrix of $k$-th and $k+1$-th simplices). Therefore, in this paper we use the following definition for Hodge Laplacians:
\begin{equation}
    \mathbf L_k=\mathbf B_k^*\mathbf B_k + \mathbf B_{k+1} \mathbf B_{k+1}^*
\end{equation}
where $\mathbf B_k^*=\mathbf B_k^T$ is the adjoint of $\mathbf B_k$, which is equivalent to the transpose of $\mathbf B_k$ in the Hilbert space. Specifically, when $k=0$, $\mathbf L_0$ is exactly the graph Laplacian $\mathbf L_0=\mathbf D-\mathbf A$. 


Furthermore, the Hodge Laplacian of the entire $K$-dimensional simplicial complex $\mathcal K$ is a block diagonal matrix $\mathbf L(\mathcal K)$, with the $k$-th block being $\mathbf L_k(\mathcal K)$ for $k=0,\dots,K$. If $\delta$ is the exterior derivative of a finite abstract simplicial complex $\mathcal K$, then
\begin{equation}
    \mathbf L(\mathcal K)=\mathbf D^2=(\delta+\delta^*)^2=\delta \delta^* +\delta^*\delta
\end{equation}
where $\mathbf D=\delta+\delta^*$ is the Dirac matrix.

\citet{HodgeRandomWalk} further defines the inter-order Hodge Laplacian for a $K$-order simplicial complex $\mathcal K$, denoted as $\mathcal L_{0:K}(\mathcal K)$. This generalized inter-order Hodge Laplacian contains information (Hodge Laplacians, boundaries, coboundaries) for all $k\in[K]$, which is defined as:

\begin{small}
\begin{equation}
    \mathcal L_{0:K}(\mathcal K)=\begin{bmatrix}\mathbf L_0 & \mathbf B_1 \\ \mathbf B_1^T & \mathbf L_1 & \mathbf B_2 \\ & ... & ... & ...\\ & & ... & ... & ...\\ & & & \mathbf B_{K-1}^T & \mathbf L_{K-1} & \mathbf B_K \\ & & & &  \mathbf B_K^T & \mathbf L_K \end{bmatrix}
\end{equation}
\end{small}

$\mathcal L_{0:K}(\mathcal K)$ contains information for all $k\leq K$ order simplices, which is a block matrix with $\mathbf L_k$ in the $k$-th diagonal block, $\mathbf B_{k}^T$ and $\mathbf B_{k+1}$ in the offset $\pm 1$ diagonal blocks, while all other blocks are zeros.

As discussed in \citep{HodgeRandomWalk}, a number of previous works such as \citep{MPSimplicialN} can be reformatted and unified by $\mathcal L_{0:K}$. We will show that simplicial transformer encoding $\mathcal L_{0:K}$ as attention bias can generalize simplicial convolutions in the next subsection. 

In addition, we can make use of $\mathcal L_{0:K}^r$ to build random walk-based positional encoding for all simplices in the $K$-dimensional simplicial complex that contains more information than random walks within the same order simplices. Similar to \citep{CWNetworks}, we can also introduce any form of local structure (such as rings and cycles) as expanded complex cells and perform random walk on them, which can greatly facilitate graph learning by incorporating higher order structures other than simplicial complexes (nodes, edges, triangles, and four-cliques).




\subsection{Simplicial Transformer with Global Attention}

As discussed in \cref{SubsecSimplicialNetworks}, although there are some networks defined on simplices and simplicial complexes, most of them are in convolutional or message passing manners~\citep{MPSimplicialN, ConvolutionalLearningOnSimplicial}. Other attention-based models all restrict their attention reception field within the range of boundaries, co-boundaries, upper and lower adjacent neighbors~\citep{SimplicialAttentionNet, SimplicialAttentionNeuralNet}, which can be viewed as weighted convolutional simplicial networks. 

In this subsection, we propose our novel full simplicial transformer, theoretically analyze the pros and cons brought by global reception field, and systematically discuss the design spaces of simplicial transformers.


\paragraph{Full Simplicial Transformer with Global Attention.}

The simplicial transformer with global attention is defined as follows.

\begin{definition}
    Denote $k$-simplices of an abstract simplicial complex $\mathcal K$ as $S_k(\mathcal K)$, whose corresponding features are $\mX\in\mathbb R^{|S_k(\mathcal K)|\times d}$. The $k$-dimensional simplicial self-attention is defined as $\mathcal {AS}_{k}$:
    \begin{equation}
        \mathcal {AS}_{k}(S_k(\mathcal K))={\rm softmax} \Big(\mX Q (\mX K)^\top \Big) \mX V
    \end{equation}
\end{definition}

where we use the notation $\mathcal {AS}_k$ to differentiate the transformer defined on $k$-simplices from the one defined on $k$-tuples in our previous sections. For simplicity, both the $k$-dimensional simplicial self-attention and the simplicial transformer (i.e. self-attention along with components such as residual connections and output MLPs) are denoted as $\mathcal {AS}_k$ due to their mild difference.

The above $\mathcal {AS}_k$ involves only $k$-simplices, yet it is natural to extend the input tokens to all simplices with order $k\leq K$ for a $K$-dimensional simplicial complex $\mathcal K$.

\begin{definition}
    Denote the set of all $k$-simplices of a $K$-dimensional abstract simplicial complex $\mathcal K$ where $k\leq K$ as $S_{0:K}(\mathcal K)$, whose corresponding features are $\mX \in \mathbb R^{\sum_{k=0}^K |S_k(\mathcal K)| \times d}$. The $0:K$-dimensional simplicial self-attention is defined as $\mathcal {AS}_{0:K}$:
    \begin{equation}
        \mathcal {AS}_{0:K}(S_{0:K}(\mathcal K))={\rm softmax} \Big(\mX Q (\mX K)^\top \Big) \mX V
    \end{equation}
\end{definition}

\paragraph{Hodge Laplacians as Attention Bias.}

Similar to the lost of connectivity information in standard graph transformer due to dense attention, one may observe that the dense simplicial transformer $\mathcal {AS}_k$ and $\mathcal {AS}_{0:K}$ are also unaware of the connectivity and structure information, including coboundaries, boundaries as well as upper and lower adjacent neighbors. To address this problem of standard (first order) graph transformer, a variety of models are proposed to encode structure information via different approaches, among which Graphormer~\citep{Graphormer} is a well known model. Concretely, Graphormer embeds node degrees, edge features and pair-wise shortest path distances to the $n\times n$ attention matrix. In its first order case, node degree and edge information can be summarized into the standard graph Laplacian $L_0$.

Now we generalize this design to the arbitrary-order simplicial transformer. For $k$-simplices, Hodge $k$ Laplacian $\mathbf L_k$ summarizes the upper adjacent and lower adjacent information into a $|S_k(\mathcal K)|\times |S_k(\mathcal K)|$ matrix, where $S_k(\mathcal K)$ is the number of $k$-simplices in $\mathcal K$. Therefore, $\mathcal {AS}_k$ and $\mathcal {AS}_{0:K}$ with attention biases are defined as:
\begin{align}
    \mathcal {AS}_{k}(S_k(\mathcal K))={\rm softmax} \Big(\mX_k Q (\mX_k K)^\top + \phi(\mathbf L_k) \Big) \mX_k V
\end{align}
\begin{align}\label{equation_simplicial_transformer_kernel}
    \mathcal {AS}_{0:K}(S_k(\mathcal K))={\rm softmax} \Big(\mX Q (\mX K)^\top+\phi(\mathbf L)\Big) \Big({\rm Concat}\big[(\mX_0 V^0)^\top, (\mX_1 V^1)^\top, \dots, (\mX_K V^K)^\top\big]\Big)^\top, \ k=0,\dots,K
\end{align}
where $\mX_k\in\mathbb R^{|S_k(\mathcal K)|\times d}$ is the feature of $S_k(\mathcal K)$ (i.e. the $k$-faces), $\mX\in\mathbb R^{\sum_{k=0}^K|S_k(\mathcal K)|\times d}$ is the concatenation of $\mX_k$, $\mathbf L$ is the Hodge Laplacian of $\mathcal K$, and $\mathbf L_k$ is the $k$-th order Hodge Laplacian; $\phi$ is an element-wise function, e.g. an identity function, an element-wise MLP etc. 

In addition, we also introduce an augmented Hodge Laplacian $\mathcal L_{0:K}$, see \cref{SubsecHodgeTheory} for more details. The key difference between $\mathcal L_{0:K}$ and the standard Hodge Laplacian $\mathbf L$ is that the former has co-boundary and boundary operators in the $\pm 1$ off-diagonal blocks, which are zeros in the latter.

Furthermore, by reweighting the attention using the Hodge Laplacians, we can recover the simplicial message passing networks. In particular,
\begin{align}\label{equation_simplicial_transformer_reweight}
    \mathcal {AS}_{0:K}(S_k(\mathcal K))=\Big({\rm softmax} \big(\mX Q (\mX K)^\top\big) \odot \phi(\mathbf L)\Big) \Big({\rm Concat}\big[(\mX_0 V^0)^\top, (\mX_1 V^1)^\top, \dots, (\mX_K V^K)^\top\big]\Big)^\top, \ k=0,\dots,K
\end{align}
where $\odot$ is the Hadamard product (element-wise product).
We will show the benefit of including Hodge-$k$ Laplacian in simplicial transformers in the following subsection, making connections between simplicial transformers and simplicial message passing networks, $\mathcal A_k$ defined on $k$-tuples as well as $k$-WL hierarchy. For each order $k$, the diagonal block $\mathbf L_k$ in $\mathbf L$ enable information aggregation from upper and lower adjacent neighbors, which are also $k$-simplices. However, the message cannot be passed among simplices of different orders through boundaries and coboundaries through $\mathbf L$, and $\mathcal L_{0:K}$ addresses this problem by introducing boundary and coboundary operators into off-diagonal blocks.

\subsection{Theoretical Analysis on Simplicial Transformers}

In this subsection, we present our theoretical results concerning simplicial transformers. Concretely, we establish the connections between our proposed simplicial transformers and two existing families of models: simplicial message passing networks and $k$-IGN ($k$-WL). We also put forward a spectral monotonicity result of the attention matrix of the simpilcial transformer, which gives more insights into the relationship between simplicial complex-based models and tuple-based models.

\paragraph{Simplicial Transformer Generalizes Simplicial Message Passing Networks.} We show that our simplicial transformer with global attention and Hodge Laplacian as attention bias is a more general version of simplicial networks in message-passing or convolution manners.

\begin{theorem}\label{TheoremSimplicialT>=MPSN}
    Simplicial transformers $\mathcal {AS}_{0:K}$ reweighted by augmented Hodge Laplacian $\mathcal L_{0:K}$ encodings (in \cref{equation_simplicial_transformer_reweight}) can approximate message passing and convolutional simplicial networks.
\end{theorem}

\begin{proof}
    \citet{MPSimplicialN} already proved that their Message Passing Simplicial Network and WL variant SWL generalizes convolutional simplicial networks such as \citep{SimplicialNeuralNetworks}. Now we only need to show that $\mathcal {AS}_{0:K}$ with $\mathcal L_{0:K}$ encodings as attention bias can exactly simulate the Message Passing Simplicial Network. WLOG, we first consider the update function of $k$-simplices $(0<k<K)$. Let $\mX_k\in \mathbb R^{|S_k(\mathcal K)|\times d}$ represent the feature of all $k$-faces, we construct
    \begin{align}
        \mX_{j} Q=1\vec 1_{|S_j|}, \mX_j K=\vec 1_{|S_j|}, \ j=0,\dots,K
    \end{align}

    We define the element-wise function as $\phi(\cdot)=\cdot \times \mathds 1(\cdot)$, where $\mathds 1(\cdot)=1$ if the element is in the non-zero block of $\mathcal L_{0:K}$ and $0$ otherwise.
    Define $V^k=W_k \times (\sum_{k=0}^K|S-k|)$, the update function can becomes,
    \begin{equation}
        \mathcal {AS}_{0:K}(\mX_k)\rightarrow \delta_{k-1}\mX_{k-1} W_{k-1}+ \mathbf L_k \mX_{k} W_k+ \delta_k^*\mX_{k+1} W_{k+1}
    \end{equation}

    According to the definition of $\delta_{k-1}, \mathbf L_k, \delta_k^*$, the three terms on the right recover the message from the boundaries, the lower and upper adjacent neighbors, and the co-boundaries, respectively. This is exactly the update function of MPSN, see Equations (2)-(6) in \citep{MPSimplicialN}. Therefore, our $\mathcal {AS}_{0:K}$ with $\mathcal L_{0:K}$ encodings as attention bias approximates MPSN and other convolutional simplicial networks arbitrarily well.
    
\end{proof}

While our sipmlicial transformer with attention bias has the capability to recover MPSN~\citep{MPSimplicialN} and other convolutional simplicial networks, the other direction is obviously not true due to the local reception field of the message passing scheme. Hence, taking advantage of the global nature of simplicial transformers may be a promising direction.






\paragraph{Spectral Monotonicity of Simplicial Attention.} 

Now we provide some additional results on the spectral monotonicity of the simplicial attention, which helps us better understand the behaviors of our simplicial transformer given a series of monotonic simplicial complexes $\mathcal K_1\subset \dots \subset \mathcal K_N$, thus establishing connections with transformers on tuples.

To start with, we consider two simplicial complexes $\mathcal K_1$ and $\mathcal K_2$ with $m_1<m_2$ elements respectively, where $\mathcal K_1$ is a sub-simplicial complex of $\mathcal K_2$. To make their spectra comparable, define $\lambda_i(\mathcal K_1)=0$ for $i\leq m_2-m_1$ and $\lambda_{m_2-m_1+i}(\mathcal K_1)=\mu_i(\mathcal K_1)$, where $\mu_i$ are the original $m_1$ eigenvalues of Hodge Laplacian $\mathcal L(\mathcal K_1)$ ordered in ascending order. Together, spectra of two simplicial complexes can be seen as left-padded non-descending sequences, thus comparable.

We first introduce a known result on spectral monotonicity of Hodge Laplacian.

\begin{lemma}\label{Lemma_spectralmonotonicityHodge}
\begin{equation}
    \lambda_j(\mathcal K_1)\leq \lambda_j(\mathcal K_2), \forall j\leq m_2
\end{equation}
\end{lemma}

The original proof is given in \cite{SpectralMonotonicity}. Next, we give a novel spectral monotonicity result on the attention matrix of a simplicial transformer. 

\begin{theorem}\label{Theorem_SpectralMonotonicity_Attnbias}
    Suppose $\mathcal K_1\subset \mathcal K_2$ , consider a simplicial transformer $\mathcal {AS}_{0:K}$ that uses the Hodge Laplacian $\mathbf L(\mathcal K)$ as attention bias. Suppose the following mild conditions hold: (i) projection matrix $Q=K$, which guarantees the symmetry of attention matrix; (ii) $\mX Q$ has all elements positive, which can be easily achieved via a ReLU activation. Denote the attention matrices of $\mathcal K_1$ and $\mathcal K_2$ (i.e. the $m_i\times m_i$ matrices before placed in softmax, $i=1,2$) as $\mA_i=(\mX_i Q)(\mX_i Q)^\top + \mathbf L(\mathcal K_i),\ i=1,2$ follow the spectral monotonicity: 
    \begin{equation}
        \lambda_j (\mA_1) \leq \lambda_j(\mA_2), \forall j\leq m_2
    \end{equation}
\end{theorem}

\begin{proof}
    The padding rule of the eigenvalues is the same as we stated. If necessary, the feature $\mX Q$ is allowed to be \textit{zero-padded} for the simplices in $\mathcal K_2$ but not in $\mathcal K_1$, thus $\mX_1$ is padded to the same dimension as $\mX_2$. Our proof parallels those of Kirchhoff Laplacian $L_0$~\citep{SpectralGraphTheory} and Hodge Laplacian~\citep{SpectralMonotonicity}. Suppose $G$ is a finite set of non-empty sets closed under the operation of taking finite non-empty subsets, a set $x\in G$ is then called \textit{locally maximal} if it is not contained in an other simplex. This suggests that the set $U=\{x\}$ is an open set in the non-Hausdorff Alexandroff topology $\mathcal O$ on $G$ generated by the basis formed by $U(x)=\{y\in G, x\subset y\}$. In this case the spectrum changes monotonically if we add a locally maximal simplex to a given complex~\citep{SpectralMonotonicity}. Now, since $\mathcal K_1\subset \mathcal K_2$, $\mathbf L(\mathcal K_1)\leq \mathbf L(\mathcal K_2)$ is in the Loewner partial order. Recalling that $\mathbf L(\mathcal K)=\mathbf D(\mathcal K)^2$, the following statement always holds:

    If $u:=\mathcal K_2\rightarrow \mathbb R$ is a vector, then the quadratic form of the attention matrix $\mA$ is 

    \begin{align}
        \langle u,\mA u \rangle&=\langle u, (\mX Q)(\mX Q)^\top + \mathbf L)u \rangle\\ &=\langle u, (\mX Q)(\mX Q)^\top + \mathbf D^2)u \rangle\\ &=\langle \mathbf Du, \mathbf Du\rangle + \langle (\mX Q)^\top u, (\mX Q)^\top u\rangle\\ &=||\mathbf Du||^2 + ||(\mX Q)^\top u||^2
    \end{align}

    According to the spectral monotonicity of Hodge Laplacians in \cref{Lemma_spectralmonotonicityHodge}, we already have that $||\mathbf D_1 u||^2\leq ||\mathbf D_2 u||^2$ always holds for any vector $u$. We aim to show that the second quadratic form also always increases when adding new maximal simplices. When adding the $m_2-m_1$ components that exist only in $\mathcal K_2$, we have

    \begin{equation}
        ||(\mX_2 Q)^\top u||^2-||(\mX_1 Q)^\top u||^2=\sum_{i=m_1+1}^{m_2} ||(\mX_2 Q)_iu_i||^2\geq 0
    \end{equation}

    where we use $(\mX_2 Q)_{1:m_1}=(\mX_1 Q)_{1:m_1}, (\mX_1 Q)_{m_1+1:m_2}=\vec 0$, i.e. we use zero padding. Together, $\langle u, \mA_1 u\rangle < \langle u, \mA_2 u\rangle$ always holds, indicating that adding new maximal simplices only increases the total quadratic form. Then, \textit{Courant Fischer Theorem} gives the result using $\mathcal S_k=\{V\subset \mathbb R^n, {\rm dim}(V)=k\}$

    \begin{equation}
        \lambda_k(\mA_1)={\min }_{V\in \mathcal S_k}\max_{|u|=1, u\in V}\langle u, \mA_1 u\rangle \leq {\min }_{V\in \mathcal S_k}\max_{|u|=1, u\in V}\langle u, \mA_2 u\rangle=\lambda_k (\mA_2)
    \end{equation}
    
\end{proof}

The spectral monotonicity result may give us some intuitions of the connection between simplicial transformers and tuple transformers mentioned before. As a tuple transformer always takes all tuples into account, it can be viewed as a simplicial transformer operating on all possible faces defined on $V=[n]$ (or, briefly speaking, on a complete graph), whose attention matrix always has the largest corresponding values compared to other simplicial complexes. In specific, an order-$k$ simplicial transformer computes attention between the simplicial complexes, a subset of all $k$-tuples considered by an order-$k$ tuple transformer. Consequently, their attention matrices follow the spectral monotonicity analyzed in \cref{Theorem_SpectralMonotonicity_Attnbias}.  To some extent, the computation of simplicial transformer is more sparse and stable compared with the full transformer defined on tuples, although the latter may avoid the expansion of attention matrix via breaking the two conditions of symmetric attention and positive inputs. More comparison and theoretical connections between simplicial transformers and tuple transformers are worth exploring in the future, including their theoretical expressive power and practical performance.

\paragraph{Connections with $k$-WL Hierarchy.}

Analogous to $\mathcal A_k$, a pure $\mathcal {AS}_{0:K}$ cannot effectively update its representations in the sense of distinguishing non-isomorphic graphs due to the dense and homogeneous attention mechanism. However, by incorporating Hodge-Laplacians as attention bias or taking PE/SE defined on simplicial complexes as input, the theoretical expressive power and real-world performance of simplicial transformers can both be boosted.

It is still an open question to establish complete connection between simplicial networks (either convolutional or attention-based) and $k$-WL hierarchy. However, we can give some primary results to show the benefit of using sparse $k$-simplices instead of all $k$-tuples.

\begin{theorem}\label{Theorem_SimplicialAttn>3WL}
    $\mathcal {AS}_{0:3}$ with $\mathcal L_{0:3}$ as attention bias can distinguish a pair of non-isomorphic strongly regular graph, namely Rook’s $4\times 4$ graph and the Shrikhande graph, which cannot be distinguished by $3$-WL.
\end{theorem}

\begin{proof}
    We have already shown that $\mathcal {AS}_{0:3}$ approximates the message passing simplicial network (MPSN) with order $3$ arbitrarily well. \citet{MPSimplicialN} prove that MPSN of order $3$ can distinguish Rook’s $4\times 4$ graph and the Shrikhande graph. It is a well-known fact that $3$-WL cannot distinguish strongly regular graphs.
\end{proof}

\subsection{Sparse Simplicial Transformers}

Analogously to tuple-based transformers, our simplicial transformers also benefit from various techniques including sparse attention mechanisms, thus improving both expressive power and real-world performance. 

The sparse attention mechanisms we proposed for tuple-based high order transformers can be naturally extended to simplicial transformers with only a minor modifications, where each token now becomes simplices (either of same order or of different orders, e.g. from $0$-simplices to $K$-simplices) instead of $k$-tuples. For neighbor attention, the concept of (local) $k$-neighbor for tuples now switches to a broader definition for simplices. For a $k$-order simplex , we consider its coboundaries (which are $k+1$-simplices), boundaries (which are $k-1$-simplices) and upper/lower adjacent simplices (which are $k$-simplices) as its extended 'neighbor', see \citep{MPSimplicialN} for more details. We denote the \textit{simplex neighbor attention} implementation as $\mathcal {AS}_{k_1:k_2}^{\mathsf {SN}}$, where $k_1$ and $k_2$ are the lowest and highest order of the simplices we consider. Simplex neighbor attention is sparse and capable of capturing local structures. To improve its expressiveness, different types of relations (including coboundary, boundary and upper/lower adjacent, depending on the position in Hodge Laplacian $\mathcal L_{k_1:k_2}$) are embedded to re-weight the attention matrix. As for virtual tuple attention, it is now naturally converted to virtual simplex attention, which we denote as $\mathcal{AS}_{k_1:k_2}^{\mathsf {VS}}$. The advantage of virtual simplex attention is that it can capture global information, similar to the virtual node in MPNNs. However, it may also suffer from over-smoothing or over-squashing.

\section{DISCUSSION}\label{SectionAppendixDiscussion}

\subsection{Further Discussion on Related Work}

In \cref{Section_relatedwork} we already introduce high-order transformers in \citep{PureTransformerspowerful, RepresentationalStrengthsTransformer}. In this subsection, we give some in-depth discussion on related works, mainly \citep{RepresentationalStrengthsTransformer}.

\citet{RepresentationalStrengthsTransformer} proposed another family of high-order transformers, see \cref{Section_relatedwork} for descriptions. Besides the theoretical expressive power compared with $k$-WL and $k$-FWL hierarchy mentioned in our main text, we now provide additional results from the perspective of the representation power, and particularly the ability to solve Match-$m$ problem.


\paragraph{Match-$m$ Problem.}

\citet{RepresentationalStrengthsTransformer} presents problems of \textit{pair detection} (Match$2$) and \textit{triple detection} (Match$3$), which are defined for inputs $\mX=(x_1,\dots,x_n)\in [M]^{n\times d}$ (for some $M={\rm poly}(n)$) as

\begin{equation}
    {\rm Match}2(\mX)_{i\in[n]}=\mathds {1}\Big(\exists j\ \  {\rm s.t.} \ \ x_i+x_j=\vec{0} \ \ ({\rm mod} M)  \Big)
\end{equation}

\begin{equation}
    {\rm Match}3(\mX)_{i\in[n]}=\mathds {1}\Big(\exists j_1,j_2\ \  {\rm s.t.} \ \ x_i+x_{j_1}+x_{j_2}=\vec{0} \ \ ({\rm mod} M)  \Big)
\end{equation}

\citet{RepresentationalStrengthsTransformer} conclude that a single layer of standard transformer (i.e. $\mathcal A_1$) with input and output MLPs and an $O(d)$-dimensional embedding can efficiently compute Match$2$, but fails to compute Match$3$ unless the number of heads $H$ or the embedding dimension $d_k$ grows polynomially in $n$. However, they show that a certain ``third-order tensor self-attention" (which resembles $\mathcal A_{1,2}$ in our formulation) can efficiently compute the Match$3$ problem with a single unit.

Their matching problem can be easily generalized to arbitrary order, called the Match-$m$ problem, which is defined for $\mX=(x_1,\dots,x_n)\in [M]^{n\times d}$ (for some $M={\rm poly}(n)$) as

\begin{equation}\label{equation_Match-m}
    {\rm Match}m(\mX)_{i\in[n]}=\mathds {1}\Big(\exists j_1,\dots, j_{m-1}\ \  {\rm s.t.} \ \ x_i+x_{j_1}+\dots+x_{j_{m-1}}=\vec{0} \ \ ({\rm mod} M)  \Big)
\end{equation}

Following \citep{RepresentationalStrengthsTransformer}, we allow a single blank token $x'=\vec 0$ to be appended at the end of sequence $\mX=(x_1,\dots,x_n)$, and allow the existence of a positional encoding with $x_{i,0}=i$. Thus, the input to the attention is augmented as $\mX'=(x_1,\dots,x_n,x')$. Additionally, the input can be transformed by an element-wise MLP $\phi:\mathbb R^d\rightarrow \mathbb R^m$. We next show the ability of $\mathcal A_{1,1}^{1,k}$ transformer to address Match-$m$ problem, which is a generalization of the conclusions in \citep{RepresentationalStrengthsTransformer}.
 
\begin{proposition}\label{Theorem_Match_m}
    Order-$1,m-1$ Transformer in \citep{RepresentationalStrengthsTransformer} can efficiently solve the MATCH-$m$ problem with $d\times 2^{m-1}+1$ hidden width. 
\end{proposition}


The proof essentially follows the proof of Theorem 6 and Theorem 18 in \citep{RepresentationalStrengthsTransformer}. We cuse a similar construction related to the trigonometric function, except that we need more terms.

\citet{RepresentationalStrengthsTransformer} also give an augmented variant of high-order graph transformer (Definition 8 in \citep{RepresentationalStrengthsTransformer}), which determines each element of the self-attention tensor based on both its respective inner product (attention score) and on the presence of edges among the corresponding inputs (provided by an input adjacency matrix of the graph). This enhancement somewhat resembles our sparse mechanisms, both of which improve the representation power by forcing some elements in the attention to be zero. Some results on substructure counting using their edge-augmented high-order transformers. 

\paragraph{Connection with Our Models.} The definition of their general $s$-order order transformer (Definition 7 in \citep{RepresentationalStrengthsTransformer}) strongly resembles our cross-attention $\mathcal A_{1,s-1}$. The difference lies in that they reconstruct the $s-1$-order key and value tensor from the input $1$-order via tensor product in every layer, while we explicitly maintain representations of the $s-1$-order tensor. Their method is a sort of hierarchical pooling~\citep{RelationalklWL} (they only maintain representations for $1$-order tensor), and the reconstruction may not recover full information of the original $s-1$-order tensor. As a result, their method may be weaker in expressive power and representation power. Their high-order graph transformer (Definition 8 in \citep{RepresentationalStrengthsTransformer}) is augmented with an input adjacency matrix, allowing the attention to incorporate edge information and graph structure. Our model also has the ability via: (1) initialization based on isomorphism types of tuples, and (2) sparse attention mechanism based on (local) neighbor information. It would be an interesting future direction to theoretically investigate the strict relationship between our transformers and theirs. Empirically, there are some common inspirations that can be applied to practical models of both ours and theirs, including reweighting the attention score via edge information and calculating high-order tensor representations from first-order tensors. We also implement cross attention $\mathcal A_{1,2}$, which to some extent recovers their ``$3$-order transformer'', see \cref{SectionExperimentsAppendix} for more details.


\subsection{Discussion on Other Theoretical Properties}

\paragraph{Comparison Between Transformer and MPNN/WL/IGN.}

For standard transformer and MPNN, it is obvious that standard transformers without special design are not aware of structure information, while MPNN does via the edge information. Instead, MPNN can approximate the linear transformer with a virtual node \citep{ConnectionMPNNGT}. 

The same holds for higher-order cases. Higher order (simplified) transformer without equivalence class basis (and the designs in \citep{PureTransformerspowerful}) is homogeneous, which means that all tokens play equal roles in updating one token. The only input of structural information is through the initialization of $k$-tuple representations. Instead, $k$-WL and $k$-IGN can distinguish different classes of $k$-tuples with the internal design of $k$-neighbor and equivalence class basis, respectively. Essentially, these designs are based on indices and are only dependent on the order $k$ (or possibly the number of nodes). Each $k$-IGN layer alone is capable of aggregating all $k$-tuples' information, while $k$-WL does so in $[k/2]+1$ steps (and simulates one $k$-IGN layer in $k-1$ steps). As a consequence, without 'virtual node' $k$-IGN can still capture global information as a kernelized high-order transformer (e.g. $k$-Performer), and it has even more equivalence classes other than the sum of all $k$-tuples. Thus, the key difference between $k$-transformers and $k$-IGN lies in the pairwise multiplication in attention computation and the existence of  equivalence class basis.

It is remarkable that deepsets for one dimension and high-order functions have internal gap: graph isormorphism problems cannot be reformed into deepsets, but can be reformatted to the problem on $k>=2$ dimensions, yet $k$-IGN and $k$-WL cannot solve them. Therefore, $k$-IGN are not universal approximators on $k$-dimensional tensors, while transformers cannot even approximate $k$-IGN without input indices or equivalence class basis (dense connection lose expressivity).

\paragraph{Computation Complexity.} $2$-IGN needs $n^4$ spaces, or more generally, $n^{2k}$ for equivariant linear layer from $n^k$ to $n^k$; this can be verified in Equation (9a,9b) and (10a,10b) in \cite{IGN} . In other words, each $k$-tuple receives information from all $n^k$ tuples, which highly resembles the computation in the transformer $\mathcal A_k$. As a matter of fact, these two architectures have computational complexities of the same magnitude. Again, the key difference between $k$-transformers and $k$-IGN lies in the pairwise multiplication in attention computation and the existence of equivalence class basis. Attention may lead to stronger performance in real-world tasks, similar to the SOTA performance in CV and NLP; yet the equivalence class basis in $k$-IGN is crucial to distinguish non-isomorphism graphs. By introducing equivalence class basis (or equivalently, adjacency relations of $k$-tuples), $\mathcal A_k$ can be as power as $k$-IGN but cannot surpass the latter since attention does not increase expressive power.

\paragraph{Tensor Multiplication Ability.} $2$-PPGN can simulate matrix multiplication and \citet{PPGN} actually generalize it to any order $k$, so $k$-PPGN (as powerful as $k$-FWL or $k+1$-WL) can implement a certain class of $k$-order tensor multiplication. However, this is different from the column-wise Kronecker product in \citep{RepresentationalStrengthsTransformer} and the normal matrix multiplications. 

\paragraph{Separation Speed.} An interesting thing is that the tensor multiplication ability does not necessarily increase the expressive power. As an example for order 2, \citet{WalkMPNNSecondOrder} propose a walk-MPNN which is bounded by 2-PPGN and 2-FWL. Walk-MPNN allows multiple matrix (as it is for order-2 tensors) multiplication, but this does not increase its expressive power (still bounded by 2-FWL). Instead, multiple matrix multiplication only makes it distinguish graphs faster than 2-FWL. 
A similar phenomenon would be: \citet{ExpressivepowerkIGN} shows that (k-1) iterations of $k$-WL simulate one $k$-IGN layer. Therefore, $k$-IGN is faster (in the sense of distinguishing graphs or representation power, not computation efficiency) than $k$-WL due to its global computation ($k$-IGN makes use of all $k$-tuples, while $k$-WL only makes use of $k$-neighbors). Although equivalent in their expressive power, the different separation speeds of $k$-WL and $k$-IGN may lead to distinct practical performance, which also applies to our high-order transformers. Our sparse attention mechanisms resemble $k$-WL and are more 'local' than $k$-IGN. Particularly, one our neighbor attention layer is equivalent to one $k$-WL iteration, and $(k-1)$ layers of our neighbor attention is equivalent to one $k$-IGN layer with less complexity ($O(n^{k+1}k(k-1)$ for $(k-1)$ layers of neighbor attention v.s. $O(n^{2k})$ for one layer of $k$-IGN). 


\section{EXPERIMENTS AND EMPIRICAL ANALYSIS}\label{SectionExperimentsAppendix}

\subsection{Model Implementation Details}\label{SubsecImplementationAppendix}


\paragraph{Dense and Sparse Implementation.} Generally, there are two principal ways to implement our attention: dense tensor multiplication and sparse attention. The former can be based on Einstein sum or other fast tensor multiplication methods, which is applicable to all the categories of our dense attention mechanisms. The latter can be efficiently implemented via torch-geometric relevant libraries. The efficiency of dense implementation can be improved via advanced techniques, yet the general memory and time consumption are internally higher than the sparse one, which is one of the main motivations of our work.

\paragraph{Initialization.} As stated in our theoretical analysis, the initialization of tuples should depend on their isomorphism types. To implement this, we first embed node and edge features through an embedding layer. Then we concatenate the corresponding node and edge features according to the indices in the tuple in order, which implement an ordered set according to Remark 3. To keep the latent dimension consistent with the rest of the network, we apply MLPs to project the concatenated tuple features onto the common network width. We also offer an optional initialization method that simply sums up all node and edge features. Though not theoretically expressive, we observe that this implementation generally performs similarly to (or slightly weaker than) the concatenation initialization.

\paragraph{Pooling.} We implement a comprehensive family of pooling functions for tuple transformers and simplicial transformers. For tuple transformers, (a) in node-level task, we perform $k$-times of hierarchical pooling, each pooling obtains a $n^{t-1}$ tensor from the $n^t$ tensor, eventually resulting in a tensor of size $n$ corresponding to $n$ nodes; (b) in edge-level task, if the order is $2$, we directly extract the tuples containing the target edge; if the order is not $2$, we perform pooling to obtain node representations as in (a), then calculate the target edge feature according to the two node features; (c) in graph-level task, we directly operate on all tuples and support the common max/mean/add pooling methods. For simplicial transformers, as they internally contain node ($0$-simplices) features and edge ($1$-simplices) features, the pooling only extracts the corresponding simplices. The graph-level pooling is similar to tuple transformers; additionally, we support both operating on all simplices or only on $0$-simplices.

\paragraph{Attention Reweighting and Attention Bias.} In our simplicial transformer we have already described attention bias as an enhancement in terms of structure awareness, expressive power and practical performance. Consider the attention score the attention score $\mA_{\vi,\vj}$ calculated by inner-product of query token $\vi$ and key token $\vj$, both the attention reweighting and the attention bias provide additional information based on the pairwise relation of $r(\vi,\vj)$. Formally, we embed the relation with an embedding layer $\mathcal E$, then attention reweighting modifies the attention score to $\mA_{\vi,\vj} * \mathcal E(r(\vi,\vj))$, while attention bias calculates $\mA_{\vi,\vj} + \mathcal E(r(\vi,\vj))$. For instance, in (local) neighbor attention $r(\vi,\vj)$ indicates which (local) neighbor of $\vi$ is $\vj$ in (or not in); in $\mathcal A_k^{\mathsf {Ngbh+}}$, $r(\vi, \psi_j(\vi, u))={\rm adj}(\vi_j, u)$, which provides additional structural information on the connectivity between $\vi_j$ and $u$ as $\delta\text{-}k$-WL does; in simplicial transformer $r$ is just the Hodge Laplacian $\mathbf L$ or $\mathcal L_{0:K}$ to indicate boundary, coboundary and adjacent information. Similar techniques are also adopted in other graph transformers~\citep{Graphormer, GraphIBwithoutMP, Exphormer}.

\paragraph{Positional and Structural Encodings.}

Extensive literature provided evidence that positional encodings (PE) and structural encodings (SE) can significantly improve the theoretical expressive power and real-world performance of graph transformers~\citep{AttendingGT}. Similarly, we can make use of PE and SE for high-order tuples to facilitate high-order graph transformers. The first way to incorporate PE and SE into high-order graph transformers is to concatenate existing PE and SE for nodes and edges (e.g. RWSE~\citep{RWSELearnable} and LapPE~\citep{RethinkingGTLap}) to compose PE/SE for tuples or simplicial complexes. Another way is to design novel PE/SE for high-order structures from scratch, yet current PE and SE for general $k$-tuples are limited. In comparison, PE and SE for $k$-simplicial complexes are better studied and reveal more elegant properties~\citep{HodgeRandomWalk}. Following GPS~\citep{GPS}, we choose our PE and SE exactly the same as GPS (including LapPE and RWSE).

\paragraph{Third-Order Variants.} Since there are $n^3$ many $3$-tuples, including all tuples $\mathcal A_3$ would result in $O(n^6)$ complexity regarding $n$, which is obviously not practical. Even local neighbor attention $\mathcal A_3^{\mathsf {LN}}$ has $O(n^3 \bar D)$ complexity, where $\bar D$ is the average node degree. Consequently, we implement $\mathcal A_3$ with sampling tuples. Instead of random sampling, we want to take graph structure into account; thus we sample those \textit{connected} $3$-tuples. On average there are $O(n \bar D^2)$ connected $3$-tuples, therefore even dense attention would not exceed $O(n^2 \bar D^4)$ complexity. However, note that the dense attention again does not provide the connectivity of these sampled $3$-tuples, which may have a negative impact on performance. We report the results of sampling connected $3$-tuples in \cref{Table_crossatten_A12}. 

\paragraph{Cross Attention.} We implement cross attention $\mathcal A_{1,2}$ for ablation study. The first-order tensors are updated according to the definition; yet, we also want to update representations of second-order tensors. Inspired by \citep{RepresentationalStrengthsTransformer} which reconstructs high-order tensors from first-order tensor, we modify $\mathcal A_{1,2}$ as follows. Denote the first-order query tensor $\mX \in \mathbb R^{n}$, second-order key tensor $\mY\in \mathbb R^{n^2}$, and $e_{ij}$ as the feature of edge connecting nodes $i,j$. For each layer $l$ the update procedure is as follows,
\begin{align}
    \mX^{(l+1)}&=\mathcal A_{1,2}^l(\mX^l, \mY^l)\\
    e_{ij}^{(l+1)}&=\psi_e^l(e_{ij}^l, \mX_i^{(l+1)}, \mX_j^{(l+1)})\\
    \mY_{i,j}^{(l+1)}&=\psi_v^l(\mY_{i,j}^l, \mX_i^{(l+1)}, \mX_j^{(l+1)}, e_{ij}^{(l+1)})
\end{align}
where $\psi_e,\psi_v$ are MLPs, and superscripts $^{(l)}$ refer to the $l$-th layer. The experimental results of $\mathcal A_{1,2}$ are reported in \cref{Table_crossatten_A12}.

Furthermore, inspired by our sparse attention mechanism for self-attention, we also implement dense (original) and sparse $\mathcal A_{1,2}$. In the dense version, all $n$ query tokens compute attention with all $n^2$ key tokens, resulting in $O(n^3)$ complexity. In the sparse version (denoted as $\mathcal A_{1,2}^{\mathsf {Ngbh}}$), every query token only computes attention with those $2$-tuples that contain the query node itself, and hence enjoy only $O(n^2)$ complexity. We experimentally find that the sparse $\mathcal A_{1,2}$ is not only more efficient in running time, but also usually demonstrates better performance than the dense version. This observation again verifies our motivation: our sparse attention not only reduces computation complexity, but also introduces structure information which improves model performance.

\subsection{Dataset Description and Experiments Overview}

\begin{table}[t]
\caption{Overview of the graph learning datasets used in the paper}
\label{Table_dataset_statistics}
\vskip -0.1in
\begin{center}
\begin{small}
\resizebox{1.\columnwidth}{!}{
\begin{tabular}{cccccccc}
\toprule
\multirow{2}{*}{Dataset} & \multirow{2}{*}{\#Graphs} & Avg. \# & Avg. \# & \multirow{2}{*}{Directed} & Prediction & Prediction & \multirow{2}{*}{Metric} \\
 & & nodes & edges & & level & task & \\
\midrule 
Edge detection & 12,000 & 23.2 & 24.9 & No & edge  & binary classif. & Accuracy\\
CSL & 150 & 41 & 82 & No & graph & 10-way classif. & Accuracy \\
Substructure counting & 5000 & 18.8 & 31.3 & No & graph & regression & Mean Abs. Error \\
\midrule
WebKB-Cornell & 1 & 183 & 298 & Yes & node & 10-way classif. & Accuray\\
WebKB-Texas & 1 & 183 & 325 & Yes & node & 10-way classif. & Accuray\\
WebKB-Wisconsin & 1 & 251 & 515 & Yes & node & 10-way classif. & Accuray\\
\midrule
ZINC & 12,000 & 23.2 & 24.9 & No & graph  & regression & Mean Abs. Error\\
Alchemy & 202,579 & 10.0 & 10.4 & No & graph & regression & Mean Abs. Error\\
\midrule
ogbg-molhiv & 41,127 & 25.5 & 27.5 & No & graph  & binary classif. & AUROC \\
\midrule
Peptides-func & 15,535 & 150.9 & 307.3 & No & graph  & 10-task classif. & Avg. Precision\\
Peptides-struct & 15,535 & 150.9 & 307.3 & No & graph  & 11-task regression & Mean Abs. Error\\

\bottomrule
\end{tabular}
}
\end{small}
\end{center}
\vskip -0.1in
\end{table}

We provide dataset statistics in \cref{Table_dataset_statistics}. Our experiments include both synthetic and real-world datasets. 

For synthetic datasets, we have already provided the results of edge detection (derived from ZINC~\citep{Zinc}, see \citep{AttendingGT} for more details) and CSL in the main text. We provide an additional substructure counting task in \cref{subsecexpsimulation}, since counting substructure is another important capability of graph learning models in addition to expressive power. 

For real-world datasets, we conduct extensive experiments on various datasets and report the results in \cref{subsecexprealworld}. We have already reported comprehensive results of almost all our models on ZINC~\citep{Zinc} in our main text. Additional results include node classification tasks in WebKB~\citep{WebKB}, graph-level regression task on Alchemy~\citep{alchemy}, graph-level classification task on OGBG-molhiv~\citep{OGB} and two datasets from Long Range Graph Benchmark (LRGB)~\citep{LRGB}.

As emphasized in our main text, since our main goal is to verify the scalability and empirical benefits of our models, we directly adopt experimental settings and hyper-parameters from the survey paper \citep{AttendingGT} for simulation tasks, and from the GraphGPS~\citep{GPS} paper (a recent SOTA and popular MPNN + graph transformer baseline) for real-world tasks. Even if we do not perform hyperparameter search, our models still reveal highly competitive performances. For computing infrastructure, all our experiments are carried out with NVIDIA GeForce RTX 3090 and 4090. Code available at \href{https://github.com/zhouc20/k-Transformer}{https://github.com/zhouc20/k-Transformer}.




\subsection{Simulation Results}\label{subsecexpsimulation}

The ability of substructure counting is an important perspective in measuring a graph learning model. The ability to count substructures has close connections to theoretical expressive power, but also not completely equivalent. Substructure counting ability may have more detailed and complex results, as well as more practical impacts on real-world tasks. For example, in molecular graphs detecting cycles and rings is crucial in predicting molecule properties.

\paragraph{Substructure Counting.} We provide the results of substructure counting in \cref{Table_counting}. The dataset is derived from \citep{CountSubstructures} which contains five thousand random regular graphs. There are four target substructures: triangle, tailed triangle, star and chordal-cycle. We measure the performance by MAE, lower MAE corresponding to stronger substructure counting ability. Note that when the model has extremely small MAE ($\sim 0.01$), we regard it capable of counting the target substructure and ignore the difference in absolute MAE value at such a low magnitude. Such a small difference does not necessarily reflect comparison among models - the MAE may depend more on training configurations and hyper-parameters instead of models (and that difference models may vary in their best configurations, so the comparisons are not completely fair).

For baseline models, we choose GCN~\citep{GCN}, GIN~\citep{HowPowerfulGNNGIN} PNA~\citep{PNA} as the baseline models that are basically equivalent to $1$-WL, and PPGN~\citep{PPGN} that has provable $3$-WL equivalent expressive power. We also report the performance of Transformer $\mathcal A_1$, i.e. the commonly used plain transformer. 

It is known that $1$-WL cannot count triangles, tailed triangles and chordal-cycles, but it can count stars~\citep{CountSubstructures}. With their theoretical expressive power upper-bounded by $1$-WL, we observe that GCN, GIN and PNA all have relatively huge loss while counting triangles, tailed triangles and chordal-cycles; only GIN with provable $1$-WL expressivity can count stars with a small MAE. PPGN can count all substructures well with low MAE ($\sim 0.01$), however at the cost of $O(n^3)$ complexity. The poor performance of the plain transformer $\mathcal A_1$ is not surprising: its MAE is almost two magnitudes higher than PPGN, which indicates that $\mathcal A_1$ cannot count any substructures. This aligns perfectly with our theory that $\mathcal A_1$ is strictly less expressive than $1$-WL.

Interestingly, the results of our proposed models also align quite well with our theories, which verify the effectiveness of our sparse high-order transformers.
\begin{itemize}
    \item We already theoretically proved that $\mathcal A_2^{\mathsf {Ngbh}}$ is as powerful as $2$-WL (hence $1$-WL as well as GIN). The performance of our $\mathcal A_2^{\mathsf {Ngbh}}$ model highly resembles that of $1$-WL equivalent GIN for all tasks, which is consistent with the theory. $\mathcal A_2^{\mathsf {Ngbh}}$ even achieves SOTA on counting stars, which implies the empirical benefit of bringing in global information (recall the definition of $2$-neighbor) - although it does not increase theoretical expressive power, the global information helps in algorithmic alignment.
    \item $\mathcal A_2^{\mathsf {Ngbh+}}$ is proved to be strictly more powerful than $\mathcal A_2^{\mathsf {Ngbh}}$, which is verified by the experimental results. When counting triangles, tailed triangles and chordal-cycles that $\mathcal A_2^{\mathsf {Ngbh}}$ fails, $\mathcal A_2^{\mathsf {Ngbh+}}$ all complete perfectly (the same MAE magnitude as PPGN). This verifies that $\mathcal A_2^{\mathsf {Ngbh+}}$ is strictly more expressive than $2$-WL, otherwise it would not be able to count any of triangles, tailed triangles and chordal-cycles.
    \item In \cref{SectionProofAppendix} we show that $\mathcal A_2^{\mathsf {LN}}$ is strictly less expressive than $\mathcal A_2^{\mathsf {Ngbh+}}$, which is attributed to the fact that $\mathcal A_2^{\mathsf {LN}}$ has slightly higher MAE compared to $\mathcal A_2^{\mathsf {Ngbh+}}$ in all four tasks. However, $\mathcal A_2^{\mathsf {LN}}$ is much more efficient than $\mathcal A_2^{\mathsf {Ngbh+}}$ as well as PPGN, and the absolute MAE values of $\mathcal A_2^{\mathsf {LN}}$ are still at a low magnitude compared with $1$-WL baseline models. In summary, $\mathcal A_2^{\mathsf {LN}}$ strikes a good balance between expressivity and efficiency.
    \item The results of the simplicial transformer with simplex neighbor attention mechanism $\mathcal {AS}_{0:1}^{\mathsf {SN}}$ are also consistent with our theory in \cref{SectionSimplicialTransformer}. Leveraging the connectivity of $0$-simplices (nodes), $1$-simplices (edges), and partially $2$-simplices (triangles), $\mathcal {AS}_{0:1}^{\mathsf {SN}}$ performs highly competitive in counting triangles and tailed triangles. However, it fails to count chordal-cycles, which indicates that it is more powerful than $1$-WL yet does not achieve full $3$-WL expressivity. It has $O((n+m) \bar D_{\mathcal S})$ complexity where $m=|E|$ is the number of edges, while $\bar D_{\mathcal S}$ is the average number of extend neighbors (including boundaries, coboundaries, and upper/lower adjacent neighbors); hence $\mathcal {AS}_{0:1}^{\mathsf {SN}}$ is generally even more efficient than $\mathcal A_2^{\mathsf {LN}}$ on sparse graphs.
    
\end{itemize}

Again, both PPGN and most of our models are completely capable of counting these substructures. However, since PPGN needs internal $O(n^3)$ complexity, our model variants are much more efficient.

\begin{table}[t]
\caption{Substructure counting performance (test MAE $\downarrow$). Highlighted are the \textcolor{orange}{first}, \textcolor{teal}{second} and \textcolor{violet}{third} results.}
\label{Table_counting}
\vskip -0.1in
\begin{center}
\begin{tabular}{lcccc}
\toprule
Model & Triangle & Tailed triangle & Star & Chordal-Cycle\\
\midrule
GCN & 0.4186 & 0.3248 & 0.1798 & 0.2822\\
GIN & 0.3569 & 0.2373 & 0.0224 & 0.2185\\
PNA & 0.3532 & 0.2648 & 0.1278 & 0.2430\\
PPGN & \textcolor{orange}{0.0089} & \textcolor{orange}{0.0096} & \textcolor{violet}{0.0148} & \textcolor{orange}{0.0090} \\
Transformer & 1.0712 & 0.8929 & 1.3634 & 0.9682 \\
\midrule
$\mathcal A_2^{\mathsf {Ngbh}}$ & 0.3780 & 0.2764 & \textcolor{orange}{0.0133} & 0.2467 \\
$\mathcal A_2^{\mathsf {Ngbh+}}$ & \textcolor{violet}{0.0121} & \textcolor{violet}{0.0186} & \textcolor{teal}{0.0146} & \textcolor{teal}{0.0158} \\
$\mathcal{A}_2^{\mathsf {LN}}$ & 0.0165 & 0.0201 & 0.0458 & \textcolor{violet}{0.0380}\\
$\mathcal{AS}_{0:1}^{\mathsf {SN}}$ & \textcolor{teal}{0.0103} & \textcolor{teal}{0.0133} & 0.0199 & 0.1893\\
\bottomrule
\end{tabular}
\end{center}
\vskip -0.1in
\end{table}


\subsection{Additional Results on Real-World Datasets}\label{subsecexprealworld}

In addition to the highly competitive performance on ZINC (see the main text), we also provide experimental results on other real-world datasets, including ogbg-molhiv~\citep{OGB} and two datasets from Long Range Graph Benchmark (LRGB)~\citep{LRGB}. We directly adopt experimental settings and hyper-parameters from the GraphGPS~\citep{GPS}.




\paragraph{Molecular Property Prediction.} 

For classification task, we choose ogbg-molhiv dataset from the OGB benchmark~\citep{OGB}, which contains 41k molecules. The task is a binary graph-level classification to predict whether a molecule inhibits HIV virus replication or not. The performance is measured by AUROC.

We report our results in \cref{Table_ogb}. It is notable that complex models tend to suffer from overfitting on this dataset, and SOTA models usually contain manually extracted feature, e.g. CIN~\citep{CWNetworks} and GSN~\citep{GSN} both contain artificially extracted substructures. In comparison, our models completely learn from the graph structures without any additional manually crafted information. In detail, $\mathcal {A}_2^{\mathsf {LN+VT}}$, $\mathcal A_{1,2}^{\mathsf{Ngbh}}$ and $\mathcal {AS}_{0:1}^{\mathsf {SN+VS}}$ all outperform GPS, showing the empirical benefit of high-order models. In particular, the results of $\mathcal A_{1,2}^{\mathsf{Ngbh}}$ and $\mathcal {AS}_{0:1}^{\mathsf {SN+VS}}$ are highly competitive, which implies that global information and local structures may both play an important role in the prediction of graph-level properties.

\begin{table}[t]
\caption{Results on ogbg-molhiv~\citep{OGB}. Shown is the mean $\pm$ std of 5 runs with different random seeds. Highlighted are the \textcolor{orange}{first}, \textcolor{teal}{second} and \textcolor{violet}{third} results.}
\label{Table_ogb}
\vskip -0.1in
\begin{center}
\begin{tabular}{lc}
\toprule
Model & AUROC $\uparrow$ \\
\midrule
GCN+virtual node & $0.7599\pm 0.0119$ \\
GIN+virtual node & $0.7707\pm 0.0149$ \\
PNA & $0.7905\pm 0.0132$\\
DGN~\citep{DGN} & $0.7970\pm0.0097$\\
CIN & \textcolor{orange}{$0.8094\pm 0.0057$}\\
GIN-AK+ & $0.7961\pm0.0119$ \\
GSN~\citep{GSN} & \textcolor{teal}{$0.8039\pm0.0090$}\\
\midrule 
SAN & $0.7785\pm 0.2470$\\
GPS & $0.7880\pm 0.0101$ \\
\midrule
$\mathcal {A}_2^{\mathsf {LN+VT}}$ (ours) & $0.7901\pm 0.0082$\\
$\mathcal A_{1,2}^{\mathsf{Ngbh}}$ (ours) & \textcolor{violet}{$0.7981\pm 0.0097$} \\
$\mathcal {AS}_{0:1}^{\mathsf {SN+VS}}$ (ours) & \textcolor{violet}{$0.7981\pm 0.0114$} \\

\bottomrule
\end{tabular}
\end{center}
\vskip -0.1in
\end{table}

\begin{table*}[ht]
\caption{Experiments on two datasets from long-range graph benchmarks (LRGB)~\citep{LRGB}. Shown is the mean $\pm$ std of 5 runs with different random seeds. Highlighted are the \textcolor{orange}{first}, \textcolor{teal}{second} and \textcolor{violet}{third} results.}
\label{Table_LRGB}
\vskip 0.15in
\begin{center}
\begin{tabular}{lcc}
\toprule
model & Peptides-func (AP $\uparrow$)  & Peptides-struct (MAE $\downarrow$) \\
\midrule
GCN & $0.5930\pm 0.0023$ & $0.3496\pm 0.0013$ \\
GINE & $0.5498\pm 0.0079$ & $0.3547\pm 0.0045$ \\
GatedGCN & $0.5864\pm 0.0077$ & $0.3420\pm 0.0013$ \\
\midrule
Transformer+LapPE & $0.6326\pm 0.0126$ & $0.2529\pm 0.0016$\\
SAN+LapPE & $0.6384\pm 0.0121$ & $0.2683\pm 0.0043$\\
SAN+RWSE & $0.6439\pm 0.0075$ & $0.2545\pm 0.0012$\\
GPS & \textcolor{orange}{$0.6535\pm 0.0041$} & \textcolor{teal}{$0.2500\pm 0.0005$}\\
\midrule
$\mathcal {AS}_{0:1}^{\mathsf {SN}}$ (ours) & $0.5876\pm 0.0079$ & $0.2703\pm 0.0015$ \\
$\mathcal {AS}_{0:1}^{\mathsf {SN+VS}}$ (ours) & \textcolor{teal}{$0.6486\pm 0.0063$} & \textcolor{violet}{$0.2524\pm 0.0009$}\\
$\mathcal {AS}_{0:1}\text{-}$dense+attn.bias (ours) & \textcolor{violet}{$0.6445\pm 0.0082$} & \textcolor{orange}{$0.2486\pm 0.0007$}\\

\bottomrule
\end{tabular}
\end{center}
\vskip -0.1in
\end{table*}

\paragraph{Long Range Interaction Prediction.} Empirically, one of the advantages of graph transformers over message-passing GNNs is that the former are able to capture long-range information via attention, while the latter suffer from restricted reception field. Long Range Graph Benchmark (LRGB)~\citep{LRGB} is a recently proposed benchmark to evaluate models' capacity in capturing long-range interactions within graphs. We choose two datasets, namely Peptides-func and Peptides-struct from LRGB to evaluate our models. Both datasets consist of atomic graphs of peptides. The task for Peptides-func is a multi-label graph classification into 10 nonexclusive peptide functional classes measure by average precision. The task for Peptides-struct is graph regression of 11 3D-structural properties of the peptides measured by MAE.

As shown in \cref{Table_scalability}, since these datasets contain large graphs, $\mathcal A_2$ with neighbor attention and local neighbor attention mechanisms cannot scale to them unfortunately. The results of our simplicial transformers are reported in \cref{Table_LRGB}, while performance of cross-attention and sampling connected $3$-tuples are displayed in \cref{Table_crossatten_A12} as ablation study. To the best of our knowledge, our simplicial transformers are the first second-order models that scale to LRGB datasets.

According to \cref{Table_LRGB}, our $\mathcal {AS}_{0:1}^{\mathsf {SN+VS}}$ and $\mathcal {AS}_{0:1}\text{-}$dense+attn.bias both outperform other transformers with PE/SE, including Transformer ($\mathcal A_1$) and SAN~\citep{RethinkingGTLap}. Interestingly, the virtual simplex in $\mathcal {AS}_{0:1}^{\mathsf {SN+VS}}$ leads to a significant performance gain compared with $\mathcal {AS}_{0:1}^{\mathsf {SN}}$, indicating that global information aggregated by the virtual simplex indeed helps model to predict long-range interactions and global properties better. $\mathcal {AS}_{0:1}\text{-}$dense+attn.bias internally capture global information due to the dense attention mechanism, and the (extended) Hodge Laplacian as attention bias provides beneficial structural information - it even surpasses GPS on the Peptides-struct dataset.

\subsection{Ablation Study and Empirical Analysis}

\subsubsection{Scalability and Efficiency Analysis}

\begin{table}[t]
\caption{Scalability of proposed models on datasets of different scales (Avg. \# nodes and Avg. \# edges are listed in \cref{Table_dataset_statistics}). Hyperparmeters are the same as GPS~\citep{GPS} for all models.}
\label{Table_scalability}
\vskip -0.1in
\begin{center}
\begin{tabular}{lccc}
\toprule
Model & ZINC & ogbg-molhiv & Peptides-func/struct\\
\midrule
$\mathcal A_2$-dense & \XSolidBrush & \XSolidBrush & \XSolidBrush \\
$\mathcal A_2$-Performer & \Checkmark & \XSolidBrush & \XSolidBrush \\
$\mathcal A_2^{\mathsf {Ngbh}}$ & \Checkmark & \XSolidBrush & \XSolidBrush \\
$\mathcal A_2^{\mathsf {LN}}$ & \Checkmark & \Checkmark & \XSolidBrush \\
$\mathcal A_2^{\mathsf {VT}}$ & \Checkmark & \Checkmark & \Checkmark \\
\midrule 
$\mathcal A_{1,2}^{\mathsf {Ngbh}}$ & \Checkmark & \Checkmark & \Checkmark \\
\midrule
$\mathcal {AS}_{0:1}$-dense & \Checkmark & \Checkmark & \Checkmark \\
$\mathcal {AS}_{0:1}$-Performer & \Checkmark & \Checkmark & \Checkmark \\
$\mathcal {AS}_{0:1}^{\mathsf {SN}}$ & \Checkmark & \Checkmark & \Checkmark \\
$\mathcal {AS}_{0:1}^{\mathsf {VS}}$ & \Checkmark & \Checkmark & \Checkmark \\

\bottomrule
\end{tabular}
\end{center}
\vskip -0.1in
\end{table}

\begin{table}[t]
\caption{\# parameters and running time (s/epoch) of our models on ZINC. Other hyper-parameters are the same as in \citep{GPS}.}
\label{Table_time}
\vskip -0.1in
\begin{center}
\begin{tabular}{lcc}
\toprule
Model & \# parameters & running time (s/epoch) \\
\midrule
GPS~\citep{GPS} & 423,717 & 21\\
\midrule
$\mathcal A_2$-Performer & 863,269 & 107 \\
$\mathcal A_2^{\mathsf {Ngbh+}}$ & 322,613 & 69\\
$\mathcal A_2^{\mathsf {LN+VT}}$ & 389,173 & 37 \\
\midrule
$\mathcal A_{1,2}$-dense & 607,253 & 64 \\
$\mathcal A_{1,2}^{\mathsf {Ngbh}}$ & 601,253 & 37 \\
\midrule
$\mathcal {AS}_{0:1}$-dense & 342,813 & 35 \\
$\mathcal {AS}_{0:1}^{\mathsf {SN+VS}}$ & 340,069 & 24 \\
\bottomrule
\end{tabular}
\end{center}
\vskip -0.1in
\end{table}

To verify the scalability and efficiency of our proposed models, we summarize the scalability of different models to various datasets in \cref{Table_scalability}, and report the number of parameters and the running time of our models on ZINC in \cref{Table_time}.

As summarized in \cref{Table_dataset_statistics}, the small molecular graphs in ZINC~\citep{Zinc} have the smallest average size, while the Peptides-func/struct contain large graphs with hundreds of nodes and edges. Interestingly, $\mathcal A_2$ with different attention mechanisms form a step-like scalability in \cref{Table_scalability}, which verifies the fact that $\mathcal A_2$-dense, $\mathcal A_2^{\mathsf {Ngbh}}$, $\mathcal A_2^{\mathsf {LN}}$ and $\mathcal A_2^{\mathsf {VT}}$ gradually reduce their complexity.

In \cref{Table_time}, we provide the number of parameters as well as the running time of our models on ZINC. As mentioned in the main text, all these models have similar number of parameters compared with GPS when sharing common hyper-parameters (e.g. number of layers and network width). Regarding running time, all our models have the same magnitude as GPS, which is completely acceptable. Remarkably, our sparse attention mechanisms significantly improve efficiency. Simplicial transformers are generally more efficient than tuple-based transformers. All these observations are consistent with our theoretical analysis.

\subsubsection{Cross Attention and Sampling Strategy}

In \cref{SubsecImplementationAppendix} we detail the implementation of cross-attention $\mathcal A_{1,2}$ and $\mathcal A_3$ with sampling connected $3$-tuples. The corresponding results are reported in \cref{Table_crossatten_A12}, where $\mathcal A_{1,2}^{\mathsf {Ngbh}}$ refers to the sparse implementation of $\mathcal A_{1,2}$ in which each first-order query token computes attention with the $2$-tuples containing the query node (i.e., $2$-neighbors). We observe that the performances vary from datasets. Specifically, $\mathcal A_{1,2}^{\mathsf {Ngbh}}$ achieves competitive $0.076$ MAE, in comparison, the dense implementation of $\mathcal A_{1,2}$ has $0.091$ MAE (not reported in the table). Despite that these methods do not achieve SOTA, they open a chance of leveraging the benefits of higher-order methods in graph learning in the future. 

\begin{table}[t]
\caption{Ablation of $\mathcal A_3$ with connected $3$-tuples and cross-attention $\mathcal A_{1,2}$. Shown is mean $\pm$ std of $5$ runs with different random seeds.}
\label{Table_crossatten_A12}
\vskip -0.1in
\begin{center}
\begin{tabular}{lcccc}
\toprule
Model & ZINC (MAE $\downarrow$) & Peptides-func (AP $\uparrow$) & Peptides-struct (MAE $\downarrow$)\\
\midrule
$\mathcal A_3$(connected) & $0.138\pm 0.007$ & $0.6329\pm 0.0117$ & $0.2529\pm 0.0023$\\
\midrule
$\mathcal A_{1,2}^{\mathsf{Ngbh}}$ & $0.076\pm0.005$ & $0.6419\pm 0.0108$ & $0.2612\pm 0.0013$ \\
\bottomrule
\end{tabular}
\end{center}
\vskip -0.1in
\end{table}

\subsubsection{Node Classification and Over-smoothing}

\begin{table}[t]
\caption{Node classification performance on three heterophilic transductive datasets. Shown is the mean $\pm$ std of $10$ runs with different random seeds.}
\label{Table_node_heterophilic}
\vskip -0.1in
\begin{center}
\begin{tabular}{lccc}
\toprule
Model & Cornell & Texas & Wisconsin\\
\midrule
GCN & $53.78\pm 3.07$ & $65.95\pm 3.67$ & $66.67\pm 2.63$\\
GCN+LapPE & $56.22\pm 2.65$ & $65.95\pm 3.67$ & $66.47\pm 1.37$\\
GCN+RWSE & $53.78 \pm4.09$ &  $62.97 \pm3.21$ &  $69.41 \pm2.66$\\
GCN+DEG & $53.51 \pm2.65$ & $66.76 \pm2.72$ & $67.26 \pm1.53$ \\
\midrule
GPS(GCN+Transformer)+LapPE & $66.22 \pm3.87$ & $75.41 \pm1.46$ & $74.71 \pm2.97$\\
GPS(GCN+Transformer)+RWSE & $65.14\pm5.73$ & $73.51 \pm2.65$ &  $78.04 \pm2.88$ \\
GPS(GCN+Transformer)+DEG & $64.05\pm2.43$ & $73.51 \pm3.59$ &  $75.49 \pm4.23$ \\
\midrule
Transformer+LapPE & $69.46\pm1.73$ & $77.84 \pm1.08$ &  $76.08 \pm1.92$ \\
Transformer+RWSE & $70.81\pm2.02$ & $77.57 \pm1.24$ &  $80.20 \pm2.23$ \\
Transformer+DEG & $71.89\pm2.48$ & $77.30 \pm1.32$ &  $79.80 \pm0.90$ \\
\midrule
Graphormer+DEG& $68.38\pm1.73$ & $76.76 \pm1.79$ &  $77.06 \pm1.97$ \\
Graphormer+attn.bias & $68.38\pm1.73$ & $76.22 \pm2.36$ &  $77.65 \pm2.00$ \\
\midrule 
$\mathcal {AS}_{0:1}$-dense+attn.bias & $70.27\pm 2.96$ & $76.84\pm 1.66$ & $77.45\pm 0.98$\\

\bottomrule
\end{tabular}
\end{center}
\vskip -0.1in
\end{table}

It is widely believed that MPNN suffers from problems including over-smoothing and over-squashing, while graph transformers can (partly) address these problems~\citep{AttendingGT}. However, there is still a lack of thorough and systematical analysis. It is natural to ask whether higher-order graph transformers can address these problems.

Here we primarily and empirically analyze the performance of high-order transformers in heterophilic transductive node classification tasks on heterogeneous graphs. Cornell, Texas, and Wisconsin are three popular datasets from WebKB~\citep{WebKB}. We follow the experimental settings and hyperparameters of \citep{AttendingGT}, and report the performance of two-layer simplicial transformers in \cref{Table_node_heterophilic}. We do not report tuple-based transformers since the graphs are directed, thus the concept of tuples is not applicable.

Our simplicial transformer without PE/SE significantly outperforms GPS with PE/SE, verifying the advantages of the global attention mechanism. Our models also have similar performance compared to Transformer with RWSE/LapPE, leveraging the advantage of structure awareness brought by attention bias. An interesting phenomenon is that Transformer alone perform better without GCN - adding GCN as in GPS would even decrease the test accuracy. \citet{AttendingGT} anticulated that this is due to the oversmoothing of message-passing GNNs on heterophilic graph. Our results somewhat support this explanation; however, more future work is needed to give a comprehensive answer to the question. Moreover, it is also worth exploring the empirical advantages of (high-order) transformers over MPNN and other non-transformer models.

\end{document}

%% file: math_commands.tex
\usepackage{amsmath,amsfonts,bm}

\def\eqref#1{equation~\ref{#1}}

\def\1{\bm{1}}

\def\vi{{\bm{i}}}
\def\vj{{\bm{j}}}

\def\vl{{\bm{l}}}

\def\vu{{\bm{u}}}
\def\vv{{\bm{v}}}

\def\vx{{\bm{x}}}

\def\mA{{\bm{A}}}
\def\mB{{\bm{B}}}
\def\mC{{\bm{C}}}

\def\mK{{\bm{K}}}

\def\mQ{{\bm{Q}}}

\def\mW{{\bm{W}}}
\def\mX{{\bm{X}}}
\def\mY{{\bm{Y}}}
\def\mZ{{\bm{Z}}}

\DeclareMathAlphabet{\mathsfit}{\encodingdefault}{\sfdefault}{m}{sl}
\SetMathAlphabet{\mathsfit}{bold}{\encodingdefault}{\sfdefault}{bx}{n}

